\documentclass[9.5pt,journal,cspaper,compsoc]{IEEEtran}

\usepackage{times}
\usepackage{wrapfig}
\usepackage[T1]{fontenc}
\usepackage{color}
\usepackage{array}
\usepackage{verbatim}
\usepackage{float}
\usepackage{footnote}
\usepackage{multirow}
\usepackage{amsmath}
\usepackage{amssymb}
\usepackage{graphicx}
\usepackage{subfig}
\usepackage{caption}
\usepackage{algorithm}
\usepackage{algorithmic}
\usepackage[ruled,vlined,algo2e]{algorithm2e}
\usepackage[utf8]{inputenc}
\usepackage[english]{babel}
\usepackage{amsthm}
\usepackage{booktabs}
\usepackage{adjustbox}
\usepackage{enumitem}

\usepackage{xurl}
\usepackage{cite}
\usepackage{cleveref}

\usepackage{mathtools}

\DeclareMathOperator*{\argmin}{arg\,min}

\newcommand{\bfsection}[1]{\vspace*{0.1cm}\noindent\textbf{#1}}

\newtheorem{lemma}{Lemma}
\newtheorem{defi}{Definition}

\def\eg{\emph{e.g. }}
\def\ie{\emph{i.e. }}

\def\vs{\emph{vs. }}
\def\wrt{\emph{w.r.t. }}
\def\etal{\emph{et al. }}

\ifCLASSINFOpdf

\else

\fi

\begin{document}
\title{Deep Loss Convexification for \\Learning Iterative Models}

\author{Ziming Zhang, Yuping Shao, Yiqing Zhang, Fangzhou Lin, Haichong Zhang, and Elke Rundensteiner
% <-this % stops a space
\IEEEcompsocitemizethanks{\IEEEcompsocthanksitem All the authors are with Worcester Polytechnic Institute, Worcester, MA 01609. Dr. Ziming Zhang and Yuping Shao are from Electrical \& Computer Engineering, Dr. Elke Rundensteiner and Yiqing Zhang from Data Science, Dr. Haichong Zhang and Fangzhou Lin, from Robotics Engineering.
\newline Email: \{zzhang15, yshao2, yzhang37, flin2, hzhang10, rundenst\}@wpi.edu
}
}

\markboth{IEEE Transaction on Pattern Analysis and Machine Intelligence}{Zhang \etal}

\IEEEtitleabstractindextext{%
\begin{abstract}
    Iterative methods such as iterative closest point (ICP) for point cloud registration often suffer from bad local optimality (\eg saddle points), due to the nature of nonconvex optimization. To address this fundamental challenge, in this paper we propose learning to form the loss landscape of a deep iterative method \wrt\ {\em predictions at test time} into a {\em convex-like} shape locally around each ground truth given data, namely {\em Deep Loss Convexification (DLC)}, thanks to the overparametrization in neural networks. To this end, we formulate our learning objective based on adversarial training by manipulating the ground-truth predictions, rather than input data. In particular, we propose using star-convexity, a family of structured nonconvex functions that are unimodal on all lines that pass through a global minimizer, as our geometric constraint for reshaping loss landscapes, leading to (1) extra novel hinge losses appended to the original loss and (2) near-optimal predictions. We demonstrate the state-of-the-art performance using DLC with existing network architectures for the tasks of training recurrent neural networks (RNNs), 3D point cloud registration, and multimodel image alignment. 
\end{abstract}

}

\maketitle
% \IEEEdisplaynontitleabstractindextext
% \IEEEpeerreviewmaketitle

%%%%%%%%%%%%%%%%%%%%%%%%%%%%%%%text body%%%%%%%%%%%%%%%%%%
% \IEEEraisesectionheading{}
\section{Introduction}\label{sec:introduction}

\IEEEPARstart{W}{e} often formulate a (deep) learning task as follows:
\begin{align}\label{eqn:learning}
    \min_{\theta\in\Theta}\sum_i\ell(f(x_i;\theta), \omega_i^*), 
\end{align}
where $\{(x_i,\omega_i^*)\}\subseteq\mathcal{X}\times\Omega$ denotes the $i$-th training sample with data $x_i$ and ground-truth label $\omega_i^*$, $f:\mathcal{X}\times\Theta\rightarrow\Omega$ denotes a mapping function, \eg neural network, parametrized by $\theta$ in the feasible space $\Theta$, and $\ell:\Omega\times\Omega\rightarrow\mathbb{R}$ denotes a loss function such as cross-entropy. For simplicity, here we take the regularization on $\theta$ as being embedded in $\Theta$.

At test time, with a learned $\theta^*$, we still try to minimize the loss for each test datapoint $\hat{x}$ by predicting its label $\hat{\omega}^*$ as 
\begin{align}\label{eqn:y}
    \hat{\omega}^* = \argmin_{\omega\in\Omega} \ell(f(\hat{x};\theta^*), \omega).
\end{align}

\bfsection{Iterative Methods at Test Time.}
An iterative method often refers to a specific implementation that proceeds in discrete steps and operates at each step on the result from the previous step. For the tasks such as classification with the cross-entropy loss, there exists a close-form solution, \ie $\hat{\omega}^* = f(\hat{x};\theta^*)$. However, there also may not exist close-form solutions for many other tasks, \eg (rigid) point cloud registration \cite{huang2021comprehensive} whose goal is to predict a rigid motion that aligns and fits one point cloud into a reference coordinate frame. Fig.~\ref{fig:motivation} illustrates an iterative method (\eg iterative closest point (ICP) \cite{4767965} and PRNet \cite{wang2019prnet}) for 3D point cloud registration which defines a function that maps a pair of point clouds and the previous affine transformation $M_{t-1}$ into the current transformation $M_t$ as output to minimize the registration loss. In fact, this procedure can be formulated using fixed-point iteration as follows based on Eq. \ref{eqn:y}:
\begin{align}\label{eqn:fixed-point}
    & \hat{\omega}^* = \hat{\omega}_T, \hat{\omega}_t = g(\hat{x}, \hat{\omega}_{t-1}; \theta^*), %\stackrel{\eg}{=} \Tilde{\omega}_{t-1} - \eta_t \nabla_{\omega}\ell(f(\Tilde{x};\theta^*), \Tilde{\omega}_{t-1}), \nonumber \\
    t\in[T],
\end{align}
where $g:\mathcal{X}\times\Omega\times\Theta\rightarrow\Omega$ denotes another mapping function for fixed-point iteration by minimizing Eq. \ref{eqn:y}. One possible instantiation of $g$ is using gradient descent (GD). %, $\nabla_{\omega}$ denotes the gradient operator \wrt\ argument $\omega$, and $\eta_t$ denotes the learning rate in gradient descent (GD) at the $t$-th iteration. 
Note that the close-form solutions can be taken as special cases of $g$.

\begin{figure}[t]
  \centering
   \includegraphics[width=.8\linewidth]{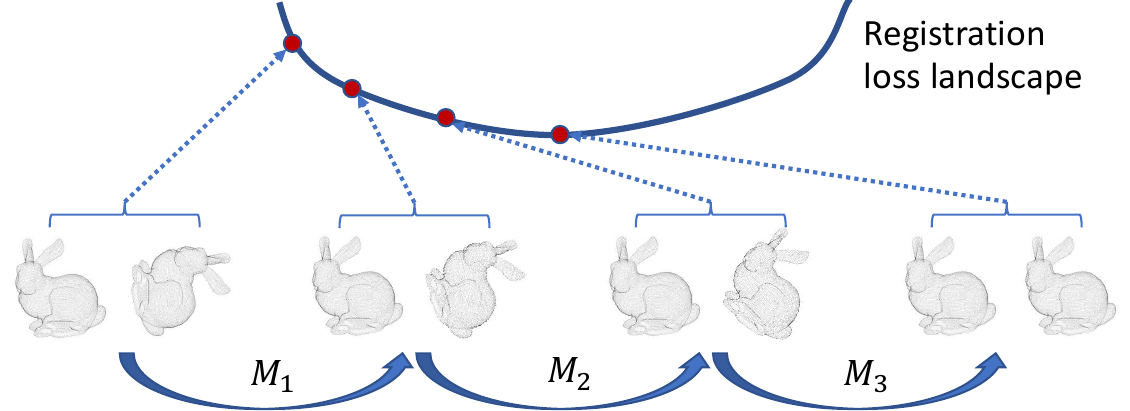}
   \caption{Illustration of iterative point cloud registration with the expectation of loss decrease on the landscape, where each $M$ denotes an affine transformation.}
   \label{fig:motivation}
   \vspace{-5mm}
\end{figure}

\bfsection{Convergence in Iterative Methods.}
Due to the nature of high nonconvexity in deep learning, the fixed-point iterations in Eq.~\ref{eqn:fixed-point} have no guarantee to converge, in general. Even if $g$ is a contraction mapping function \cite{smart1980fixed}, there is still no guarantee that the iterations will converge to the optimal solution (but often local optimality such as saddle points). Such convergence issues significantly limit the precision of iterative methods in applications. 

There are many works that address the problem of bad local optimality (see recent surveys in \cite{sun2019survey, danilova2020recent}), among which convex relaxation, a technique to relax a nonconvex problem into a convex one for global optimization, has been widely studied (\eg \cite{candes2010power}). In \cite{jain2017non}, relaxation may fail to find the solutions to the original problems because the modifications may change the problem drastically. However, if the problem possesses a certain nice {\em structure}, then under careful relaxation, the resultant relaxation will match with the one for the original problem. An example is the low-rank assumption in matrix completion, where convex relaxation can recover the exact matrices (\ie global solutions) with high probability \cite{candes2012exact, candes2010power}. Though (prior) structural information about the problems often has a significant role in nonconvex optimization, it has rarely been explored systematically in iterative methods for finding (near) optimal solutions. %, to the best of our knowledge. 

\bfsection{Convex-like Loss Landscapes \wrt\ Predictions at Test Time.}
As illustrated in Fig. \ref{fig:motivation}, constructing convex-like loss landscapes \wrt\ predictions, \ie $\omega$ in Eq. \ref{eqn:y}, at test time seems a good choice to facilitate the convergence to the optimal solution. This insight is significantly different from the literature in optimization for deep learning, where the loss landscapes \wrt\ the network weights, \ie $\theta$ in Eq. \ref{eqn:learning}, are considered. For instance, recently it has been observed that empirically the loss landscape of neural networks enjoys nice one-point convexity (a type of convex-like geometric structures centered at a single point on the loss landscape) properties locally \cite{kleinberg2018alternative}. Similarly, Li \etal \cite{li2018visualizing} observe that ``when networks become sufficiently deep, neural loss landscapes quickly transition from being nearly convex to being highly chaotic. This transition from convex to chaotic behavior coincides with a dramatic drop in generalization error, and ultimately to a lack of trainability.'' Zhou \etal \cite{zhou2018sgd} also have shown that stochastic gradient descent (SGD) will converge to the global minimum in deep learning if the assumption of star-convexity (another type of convex-like geometric structures with unimodality on all lines that pass through a global minimizer) in the loss landscapes holds. Furthermore, it has been proved in \cite{hinder2020near} that these convex-like shapes guarantee GD to converge to local minima. Though the insights on loss landscapes are different, such works highly motivate us to learn convex-like loss landscapes \wrt\ predictions at test time for iterative methods.

\bfsection{Our Approach: Deep Loss Convexification.}
Deep neural networks are often over-parametrized \cite{neyshabur2018towards}, \ie the number of parameters is more than sufficiently large to fit the training data. This insight motivates us to think about reshaping the loss landscapes of deep iterative methods as convex relaxation to search for (locally) optimal solutions. Our basic idea is to learn to form the loss landscape of an iterative method into a convex-like shape locally and approximately around each ground truth given data so that at test/inference time the iterative method will be able to converge to a {\em near-optimal} solution with some guarantee (\eg \cite{hinder2020near}). We call this procedure Deep Loss Convexification (DLC), though the loss landscape is not necessary to be strictly convex. 

In particular, we utilize star-convexity \cite{nesterov2006cubic} in our instantiation as the geometric constraint for shaping the loss landscapes, leading to extra hinge losses that are added to the original loss function. Recently star-convexity in nonconvex optimization has been attracting more and more attention \cite{lee2016optimizing, pmlr-v125-hinder20a, gower2021sgd, kuruzov2021sequential} because of its capability of finding near-optimal solutions based on GD with theoretical guarantees. Star-convex functions refer to a particular class of (typically) nonconvex functions whose global optimum is visible from every point in a downhill direction (see the formal definition in Sec. \ref{sec:method:preliminary}). From this view, convexity is a special case of star-convexity. In the literature, however, most of the works focus on optimizing and analyzing star-convex functions, while learning such functions like ours is hardly explored. At test time, the nice convergence properties of star-convexity will help find provably near-optimal solutions for the tasks.

\bfsection{Our Contributions.}
We summarize our contributions as follows:
\begin{itemize}[nosep, leftmargin=*]
    \item We propose a general framework, namely Deep Loss Convexification, to convexify the loss landscapes of iterative methods to facilitate the convergence at test time.

    \item We propose a specific learning approach based on star-convexity as the geometric constraint to mitigate the problem of local convergence in iterative methods, leading to extra hinge losses to be added to the original loss.

    \item We demonstrate state-of-the-art performance on some benchmark datasets for recurrent neural networks (RNNs), 3D point cloud registration, and multimodel image alignment. 
    
\end{itemize}

\bfsection{Differences from Our Prior Work \cite{zhang2023prise} and the Literature.}
This work is a significant extension of \cite{zhang2023prise} by proposing a novel and general framework to improve the performance of iterative models on multiple applications, with deeper theoretical analysis. In contrast to the literature, we are different from:
\begin{itemize}[nosep, leftmargin=*]
    \item \underline{\em Motivation:} We are motivated by the facts that (1) many iterative algorithms such as ICP involve nonconvex minimization that is solved based on some forms of gradient descent, and (2) convex-like loss landscapes may help converge to better solutions than the counterparts. Such reasons lead our deep loss convexification method to a unique position in the literature.

    \item \underline{\em Methodology:} We propose learning convex-like loss landscapes for iterative models by introducing star-convex constraints locally around the ground-truth {\em test-time} parameters (\ie the ones for iterative models, not network weights) into training deep neural networks (DNNs). Such convexification is the key to our success that technically distinguishes our work from others, such as Adversarial Weight Perturbation (AWP) \cite{wu2020adversarial} that tents to explicitly flatten the {\em network weight} loss landscape via injecting the worst-case weight perturbation into DNNs.

    \item \underline{\em Implementation:} We propose an efficient and effective training algorithm for network training with DLC based on random sampling, rather than gradient-based searching algorithms widely used in adversarial training.
    
\end{itemize}

\section{Related Work}

\bfsection{Loss Landscapes of Deep Neural Networks.}
Loss landscapes help in understanding the performance and behavior of deep models. \cite{li2018visualizing, bain2021lossplot, fort2019deep, engstrom2019exploring} visualize the loss landscapes to discuss the properties of networks such as convexity and robustness. Some other works try to understand the mathematical features of networks from the loss landscapes, \eg \cite{cooper2018loss, simsek2021geometry, zhang2021embedding}. A survey on this topic can be found in \cite{sun2020global}.

\bfsection{Nonconvexity \& Convexification.}
Nonconvexity is a challenging problem in the statistical learning field. Researchers find several regularized estimators \cite{tuan2000low, loh2013regularized, loh2015regularized} that can solve this issue partially. With the appearance of convexification, finding a solver for a nonconvexity problem has become a hot topic in the mathematical optimization area. Since the prevalence of deep learning neural networks, researchers try to introduce the concept of convexification \cite{yang2020graduated} into the training process. Several works \cite{wang2020adaptively, mao2016successive, reddi2018adaptive,vettam2019regularized} have shown that convexification can be utilized in training a deep neural network whose loss landscape shows nonconvexity. 

\bfsection{Contrastive Learning}
Recently, learning representations from unlabeled data in a contrastive way \cite{chopra2005learning, hadsell2006dimensionality} has been one of the most competitive research fields, \eg \cite{oord2018representation, tian2020contrastive}. Popular methods such as SimCLR \cite{chen2020simple} and Moco \cite{he2020momentum} apply the commonly used loss function InfoNCE \cite{oord2018representation} to learn latent representation that is beneficial to downstream tasks. Several theoretical studies show that contrastive loss optimizes data representations by aligning the same image's two views (positive pairs) while pushing different images (negative pairs) away on the hypersphere \cite{wang2020understanding, wang2021understanding}. In terms of applications there are a large number of works in images \cite{zimmermann2021contrastive, he2020momentum} and 3D point clouds \cite{tang2022contrastive, yang2021unsupervised}, just to name a few. A good survey can be found in \cite{le2020contrastive}.

% \bfsection{Adversarial Training.}
% Adversarial training is one of the most effective strategies for improving robustness. In domain transfer \cite{ganin2016domain}, it can help the neural network transfer knowledge from the source domain to the target domain. It can also accelerate the training as well as controlling smoothness \cite{shafahi2019adversarial, andriushchenko2020understanding, wong2020fast, miyato2015distributional}. Adversarial training is widely used with point clouds. New methods for point cloud applications \cite{liu2020adversarial, kim2021minimal,huang2022shape} are often motivated by adversarial training in image domain \cite{liu2019extending,zhang2022pointcutmix, tsai2020robust}. For instance, \cite{tu2020physically} proposes a LiDAR point rendering method to generate fake samples to fool the detector in the self-driving application, and \cite{xu2021attacking} develops feature-based and coordinate-based, norm-bounded and norm-unbounded perturbations to degrade the performance of existing point cloud semantic segmentation frameworks. Surveys can be found in \cite{bai2021recent, qian2022survey}.

\bfsection{3D Point Cloud Registration.}
Point cloud registration, \eg  \cite{xu2021omnet, li2021pointnetlk, qin2022geometric, yew2022regtr}, aims to align two or more point clouds together to build one point cloud. ICP \cite{4767965} and its variants such as SparseICP \cite{mavridis2015efficient}, robustICP \cite{phillips2007outlier}, RPM\cite{chui2003new}, TEASER \cite{yang2019polynomial} and TEASER++ \cite{yang2020teaser}, are popular methods for point matching which aim to select some key points and compute the transformation based on the selected key points. In the era of deep learning, PointNetLK \cite{aoki2019pointnetlk} uses a differentiable Lucas-Kanade for feature matching. DCP \cite{wang2019deep} uses attention-based modules for feature matching and solves the point correspondence problem through self-supervision. PCRNet \cite{sarode2019pcrnet} uses iterative passes of a feature network to aggregate transformations. PRNet \cite{wang2019prnet} utilizes selected key points and iterative passes based on actor-critic modules. RPM-Net \cite{yew2020rpm} is a deep learning based approach for rigid point cloud registration by employing RPM that is more robust to initialization. Deep global registration \cite{choy2020deep} uses a differentiable framework for pairwise registration of real-world 3D scans. Predator \cite{Huang2021CVPR} is a model for pairwise point cloud registration with deep attention to the overlap region. Nice surveys on point cloud registration can be found in \cite{huang2021comprehensive, zhang2020deep}.

\bfsection{Homography Estimation.}
Homography estimation is a classic task in computer vision. The feature-based methods \cite{ye2017robust, fu2019adaptive, ye2019fast} have existed for several decades but require similar contextual information to align the source and target images. To overcome this problem, researchers use deep neural networks \cite{erlik2017homography, nguyen2018unsupervised, le2020deep, zhang2020content} to increase the alignment robustness between the source and template images. For instance, DHM \cite{detone2016deep} produces a distribution over quantized homographies to directly estimate the real-valued homography parameters. MHN \cite{le2020deep} utilizes a multi-scale neural network to handle dynamic scenes. Since then, finding a combinatorial method from classical and deep learning approaches has become possible. Recent models such as CLKN \cite{chang2017clkn} and DeepLK \cite{zhao2021deep} focus on learning a feature map for traditional Inverse Compositional Lucas-Kanade method on multimodal image pairs. Also, IHN \cite{Cao_2022_CVPR} provides a correlation finding mechanism and iterative homography estimators across different scales to improve the performance of homography estimation without any untrainable part. A good survey can be found in \cite{agarwal2005survey}. %but the performance is still can be improved if the meaningful feature maps can be extracted by the deep learning models. Our work achieves strong performance on several Homography estimation tasks and hope it will work for other benchmarks.

\begin{figure}[t]
    % \hfill
	\begin{minipage}[b]{0.485\columnwidth}
		\begin{center}
			\centerline{\includegraphics[width=.8\linewidth, keepaspectratio,]{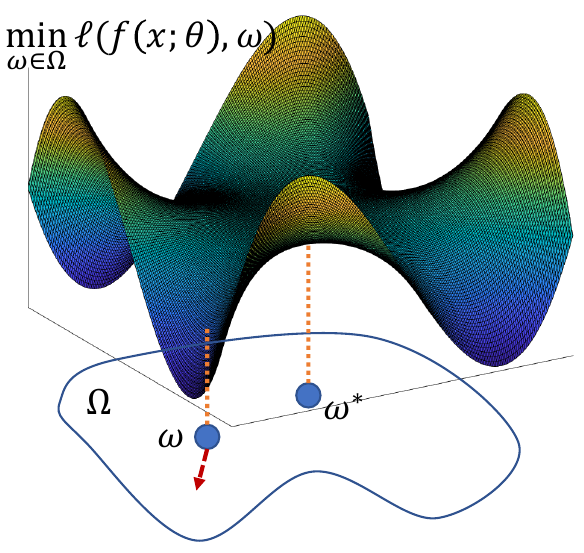}}
			\centerline{(a)}
		\end{center}
	\end{minipage}
	\hfill
	\begin{minipage}[b]{0.485\columnwidth}
		\begin{center}
			\centerline{\includegraphics[width=.8\linewidth,keepaspectratio]{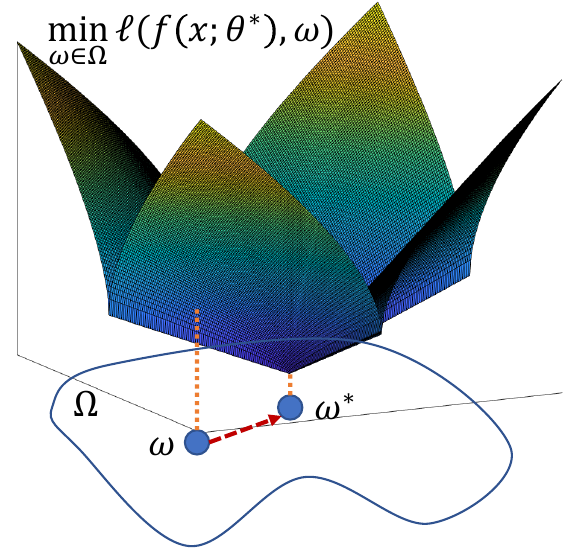}}
			\centerline{(b)}
		\end{center}
	\end{minipage}
	\vspace{-5mm}
    \caption{Illustration of differences in loss landscapes {\bf (a)} without and {\bf (b)} with convexification \wrt\ $\omega$, given data $x$.}
\label{fig:idea}
	\vspace{-5mm}
\end{figure}

\section{Method}

\subsection{Preliminaries}\label{sec:method:preliminary}

% \subsubsection{Star-Convexity}
% Below we introduce some concepts of star-convexity that will be used in our proposed approach. 
\begin{defi}[Star-Convexity \cite{lee2016optimizing}]\label{def:1}
A function $f:\mathbb{R}^n\rightarrow\mathbb{R}$ is {\em star-convex} if there is a global minimum $\omega^*\in\mathbb{R}^n$ such that for all $\lambda\in [0, 1]$ and $\omega\in\mathbb{R}^n$, it holds that
\begin{align}\label{eqn:sc_def1}
    f((1-\lambda) \omega^* + \lambda \omega) &\leq  (1-\lambda) f(\omega^*) + \lambda f(\omega).
\end{align}
% Further, if $f$ is differentiable, then Eq. \ref{eqn:sc_def1} is equivalent to the following definition:
% \begin{align}\label{eqn:qc}
%     f(\omega^*) \geq f(\omega) + \nabla f(\omega)^T(\omega^* - \omega), \forall \omega\in\mathbb{R}^n,
% \end{align}
% where $\nabla f$ denotes the (sub)gradient of $f$ and $(\cdot)^T$ denotes the matrix transpose operator.
\end{defi}

\begin{defi}[Strong Star-Convexity \cite{pmlr-v125-hinder20a}]
A differentiable function $f:\mathbb{R}^n\rightarrow\mathbb{R}$ is {\em $\mu$-strongly star-convex} with constant $\mu>0$ if there is a global minimum $\omega^*\in\mathbb{R}^n$ so that for $\forall \omega\in\mathbb{R}^n$, it holds 
\begin{align}\label{eqn:ssc_def1}
    f(\omega^*) \geq f(\omega) + \nabla f(\omega)^T(\omega^* - \omega) + \frac{\mu}{2}\|\omega^* - \omega\|^2, %\forall \omega\in\mathbb{R}^n,
\end{align}
where $\nabla$ denotes the (sub)gradient operator, $(\cdot)^T$ denotes the matrix transpose operator, and $\|\cdot\|$ denotes the $\ell_2$ norm. Note when $\mu=0$ Eq. \ref{eqn:ssc_def1}, will become equivalent to Eq.~\ref{eqn:sc_def1}.
\end{defi}

\begin{lemma}\label{lem:1}
The following conditions hold iff
a function $f:\mathbb{R}^n\rightarrow\mathbb{R}$ is $\mu$-strongly star-convex, given a global minimum $\omega^*\in\mathbb{R}^n$ and $\forall \lambda\in[0,1], \forall \omega\in\mathbb{R}^n$: %such that for all $\lambda\in [0, 1]$ and $\omega\in\mathbb{R}^n$, it holds that
\begin{align}    
    & \hspace{-3mm} f(\omega^*) \leq f(\Tilde{\omega}) - \frac{\mu}{2}\|\omega^*-\Tilde{\omega}\|^2, \label{eqn:lem-1} \\
    & \hspace{-3mm} f(\Tilde{\omega}) \leq (1-\lambda) f(\omega^*) + \lambda f(\omega) - \frac{\lambda(1-\lambda)\mu}{2}\|\omega^*-\omega\|^2, \label{eqn:lem-2}
\end{align}
where $\Tilde{\omega} = (1-\lambda) \omega^* + \lambda \omega$.
\end{lemma}
\begin{proof}
A cut through $\omega^*, \omega$ forms a convex shape if $f$ is star-convex. Therefore, since $\nabla f(\omega^*) = \mathbf{0}$, Eq. \ref{eqn:ssc_def1} will lead to Eq. \ref{eqn:lem-1} by replacing $\omega$ with $\Tilde{\omega}$ when switching the notations of $\omega^*, \omega$ in the equation. Letting $g(\omega) = f(\omega) - \frac{\mu}{2}\|\omega\|^2$, based on Eq.~\ref{eqn:ssc_def1} we have $g(\omega^*)\geq g(\omega) + \nabla g(\omega)^T(\omega^* - \omega)$, \ie $g$ is star-convex. Together with Eq. \ref{eqn:sc_def1}, we can obtain Eq.~\ref{eqn:lem-2}, and complete our proof.
\end{proof}
% \setlength{\columnsep}{10pt}%
% \begin{wrapfigure}{r}{.37\linewidth}
% \centering 
% \vspace{-5mm}
% \includegraphics[width=1\linewidth]{figures/star-convex.pdf}
% \vspace{-5mm}
% \caption{Geometric relations between $\omega^*, \omega, \Tilde{\omega}$.
% }
% \vspace{-3mm}
% \label{fig:sc-demo}
% \end{wrapfigure} 
Eq. \ref{eqn:lem-1} implies that $\omega^*$ will be a (local) minimum if it holds for $\forall \omega, \lambda$. In fact, Lemma \ref{lem:1} discusses the (tight) strong star-convexity {\em with no gradients}. In our approach we will use this lemma to incorporate the strong star-convexity as constraints in network training (see Sec. \ref{sssec:objective} for details).

\subsection{Deep Loss Convexification}

\subsubsection{Learning \& Inference Framework}
As shown in Fig. \ref{fig:idea}, by learning the network weights $\theta$ in Fig. \ref{fig:idea}(a), the loss landscape \wrt prediction $\omega$ in Fig. \ref{fig:idea}(b) is expected to be convex-like such as star-convex so that at test time the iterations can converge to the (near) optimal solution $\omega^*$.

\bfsection{Deep Reparametrization.}
The high dimensional spaces defined by deep neural networks provide the feasibility to reshape the loss landscapes. Such deep reparametrization has been used as a means of reformulating some problems such as shape analysis \cite{celledoni2022deep}, super-resolution and denoising \cite{bhat2021deep}, while preserving the properties and constraints in the original problems. This insight at test time can be interpreted as
\begin{align}\label{eqn:downstream_task}
    \min_{\omega\in\Omega} \ell(x, \omega) \xRightarrow[\text{via deep learning}]{\text{reparametrization}} \min_{\omega\in\Omega} \ell(f(x; \theta^*), \omega),
\end{align}
where $x\in\mathcal{X}$ stands for the input data, $\ell$ for a nonconvex differentiable function (\eg the Lucas-Kanade (LK) loss \cite{lucasiterative}) parametrized by $\omega\in\Omega$, $f:\mathcal{X}\times\Theta\rightarrow\mathcal{X}$ for an auxiliary function presented by a neural network with learned weights $\theta^*\in\Theta$ (\eg the proposed network in DeepLK), and $\ell_f$ for the loss with deep reparametrization (\eg the DeepLK loss \cite{zhao2021deep}). In this way, the learning problem is how to train the network so that the optimal solutions can be located using gradient descent (GD) given data.

\bfsection{Network Training.} 
    % Given a neural network for deep reparametrization, now we need to consider how to train it so that the corresponding loss landscape can satisfy the condition of $\mathcal{S}$. To this end, another loss function $\ell_s$ will be introduced to penalize the violation of the condition, leading to the following problem:    
Based on Eq. \ref{eqn:learning} and our convexification insight, it is intuitive to formulate our learning problem as 
\begin{align}\label{eqn:deep_reparametrization}
    \theta^* = & \argmin_{\theta\in\Theta}\sum_i \ell(f(x_i; \theta), \omega_i^*), \\
    & \mbox{s.t.} \; \ell(f(x_i; \theta), \omega_i)\in\mathcal{S}(x_i, \omega_i^*), \omega_i\in\mathcal{N}_{\omega_i^*}\subseteq\Omega, \forall i, \nonumber
\end{align}
where $\mathcal{S}(x_i, \omega_i^*)$ denotes a structural space, \ie a set of scalars that form a convex-like shape around ground-truth $\omega_i^*$ given data $x_i$ (different datapoints will have different locally convex-like shapes), and $\mathcal{N}_{\omega_i^*}$ denotes the neighborhood of $\omega_i^*$ in the space $\Omega$. 
% where $x_i$ is the $i$-th training data, $\omega_i^*$ is its ground truth that the iterative algorithm aims to recover, and $\theta$ is the network weights. %and $\ell_s$ returns 0 if the condition holds, otherwise positive penalties. Note that 
Ideally, it is expected that the structural constraint $\mathcal{S}$ can be met by each individual $(x_i, \omega_i^*)$ on the loss landscape. %$\ell_f(\cdot;x,\theta^*), \forall x\in\mathcal{X}$.

\bfsection{Iterative Inference.} 
We can simply apply Eq. \ref{eqn:fixed-point} for prediction, in general. Alternatively, to improve the robustness of our predictions, in our experiments we adopt average models \cite{bishop2006pattern} over iterations as our final result, defined as
\begin{align}\label{eqn:iteration_test}
        % \hspace{-2mm}
    & \hat{\omega}^* = \hat{\omega}_T, \hat{\omega}_t = \frac{1}{t}\sum_{\tau=1}^t \Delta\hat{\omega}_{\tau}, \Delta\hat{\omega}_t = g(\hat{x}, \hat{\omega}_{t-1}; \theta^*).
\end{align}
% Such average/combined models have been shown to be more robust than a single model \cite{bishop2006pattern}. %Also we show that it will be useful to control prediction errors (see Sec. \ref{sssec:analysis} for details).

% \begin{table*}[h]
% \caption{The detailed description of each data augmentation method for robustness testing.}
% % \vspace{-3mm}
% \centering
% \setlength{\tabcolsep}{9pt}{
% \begin{tabular}{lcc}
% \toprule
% Setting id& Augmentation method& Description\\
% \midrule
% Setting1& clean & without any additional change\\
% Setting2& jittering noise & adding noise under a Gaussian distribution $\mathcal{N}(0, \sigma^{2})$\\
% Setting3& scale changing & change the size under a uniform distribution $\mathcal{U}(0, a)$ with $a > 1$ \\
% Setting4& partial-points sampling & sampling points to form a partial-to-partial registration tasks\\
% \multirow{2}{*}{Setting5}& \multirow{2}{*}{half object hidding} & removing half of the points based on a centered \\
% & & plane to form an incomplete object registration tasks\\
% \bottomrule
% \end{tabular}}\vspace{-3mm}
% \label{table:augmentation method}
% \end{table*}

\subsubsection{Star-Convexification}\label{sssec:objective}

% \setlength{\columnsep}{10pt}%
% \begin{wrapfigure}{r}{.37\linewidth}
% \centering 
% \vspace{0mm}
% \includegraphics[width=1\linewidth]{figures/solution_space.pdf}
% \vspace{-5mm}
% \caption{Illustration of solution spaces induced by the constraints.}
% \vspace{-5mm}
% \label{fig:solution_space}
% \end{wrapfigure} 

\bfsection{Geometric Constraints.}
Learning exact star-convex loss landscapes with the strong condition is challenging, even in very high dimensional spaces and in local regions. One common solution is to introduce slack variables as penalty to measure soft-margins. However, this may significantly break the smoothness of star-convex loss, and thus destroy the nice convergence property of gradient descent. Therefore, in order to preserve the smoothness, we consider the geometric constraints that have capability to improve the loss landscape smoothness at different levels, just in case that only one constraint is too strong. That is,
\begin{itemize}[nosep, leftmargin=*]
    \item {\em A local minimum constraint from Def. \ref{def:1}:} On the local surface of loss landscape, the loss at the ground-truth $\omega^*$ should reach the minimum. This is also one of the requisitions for star-convexity. 
    
    \item {\em A strong star-convexity constraint from Eq. \ref{eqn:lem-1}:} This constraint implies that there should exist a quadratic lower envelope of the loss landscape with a single minimum at ground-truth $\omega^*$. %Meanwhile, it also guarantees that the loss at the ground-truth $\omega^*$ on the (local surface of) loss landscape will reach the (local) minimum, as requested by star-convexity. 
    
    \item {\em Another strong star-convexity constraint from Eq. \ref{eqn:lem-2}:} This constraint imposes strong convexity on all the curves that connect $\omega^*$ with any other point on the loss landscape.
\end{itemize}

The fundamental difference between the two strong star-convexity constraints is the positioning of $\nabla f$, where Eq. \ref{eqn:lem-1} is posited at $\omega^*$ while Eq. \ref{eqn:lem-2} is posited at $\omega$. All the three constraints can lead to their own solution spaces for the same objective, but with an overlap where solutions will better approximate star-convexity.
This insight motivates our formulation as follows.

\bfsection{Formulation.}
% Possible structural functional spaces include one-point convexity \cite{Cunha13}, quasar-convexity \cite{pmlr-v125-hinder20a, kuruzov2021sequential}, and star-convexity \cite{nesterov2006cubic, lee2016optimizing, pmlr-v125-hinder20a}. All these spaces consist of structured nonconvex functions that can guarantee gradient descent to converge to global minimizers \cite{bartlett2018gradient, pmlr-v125-hinder20a, zhou2018sgd}. In particular, 
As a demonstration purpose, we introduce star-convexity into Eq. \ref{eqn:deep_reparametrization} as the structural space $\mathcal{S}$. This naturally leads to the following constrained optimization problem with soft-margins based on Eq. \ref{eqn:deep_reparametrization}:
\begin{align} 
    \hspace{-0.5mm}\min_{\theta\in\Theta} & \sum_i\left\{ h_{\theta}(\omega_i^*) + \rho\cdot\mathbb{E}_{\omega_i\sim\mathcal{N}_{\omega_i^*}}\Big[\epsilon_{\omega_i} + \gamma_{\omega_i} + \xi_{\omega_i}\Big]\right\} \label{eqn:problem} \\
    \hspace{-2mm}\mbox{s.t.} \hspace{1.5mm} & h_{\theta}(\omega_i^*) \leq h_{\theta}(\Tilde{\omega}_i) + \epsilon_{\omega_i}, \label{eqn:con1} \\
    & h_{\theta}(\omega_i^*) \leq h_{\theta}(\omega_i) - \frac{\mu}{2}\|\omega_i^* - \omega_i\|^2 + \gamma_{\omega_i}, \label{eqn:con2} \\
    & h_{\theta}(\Tilde{\omega}_i) \leq (1-\lambda)h_{\theta}(\omega_i^*) + \lambda h_{\theta}(\omega_i) \nonumber \\
    & \hspace{20mm} - \frac{\lambda(1-\lambda)\mu}{2}\|\omega_i^* - \omega_i\|^2 + \xi_{\omega_i}, \label{eqn:con3} \\
    & \forall\epsilon_{\omega_i}\geq0, \forall\gamma_{\omega_i}\geq0, \forall \xi_{\omega_i}\geq0, \forall i, \nonumber
\end{align}
where $h_{\theta}(\cdot) \stackrel{def}{=} \ell(f(x_i;\theta), \cdot)$ denotes the loss for simplicity,  $\omega_i$ is a random sample from $\mathcal{N}_{\omega_i^*}$, $\Tilde{\omega}_i \stackrel{def}{=} (1-\lambda) \omega_i^* + \lambda \omega_i$ is a linear combination of $\omega_i, \omega_i^*$, $\epsilon_{\omega_i}, \gamma_{\omega_i}, \xi_{\omega_i}$ are three slack variables, and $\rho\geq0, \mu\geq0$ are predefined trade-off and surface sharpness parameters, respectively. Note that the integration of Eqs. \ref{eqn:con1}, \ref{eqn:con2}, \ref{eqn:con3} aims to enforce the loss landscapes to be strongly star-convex. The overlap of functional spaces defined by the three constraints greatly helps regularize the network weights.% towards reshaping the loss landscapes into star-convexity.

% \begin{figure}[t] 
% \centering 
% \includegraphics[width=1\linewidth]{figures/implementationv2.pdf}
% % \vspace{-5mm}
% \caption{Illustration of our star-convexification implementation for point cloud registration based on Eqs. \ref{eqn:problem}-\ref{eqn:con3}.} 
% % \vspace{-3mm}
% \label{fig:implementation}
% \end{figure} 

\bfsection{Contrastive Training.}
Our approach is highly related to contrastive learning since during training we try to create new fake samples $\Tilde{\omega}_i, \omega_i$ (For some applications such as point cloud registration, such fake samples can be easily generated in an online fashion with domain knowledge), compare their losses with $h_{\theta}(\omega^*)$, and finally solve our optimization problem above. The contrastive learning comes from the nature of strong star-convexity, leading to extra hinge losses that control the loss landscapes towards being star-convex. %The adversarial training starts from finding the values of $\omega_i, \lambda$ that return the maximum total hinge loss. Both together aim to control the loss landscapes towards being star-convex.

% Two key challenges in such methods are: (1) the training is not easy to converge due to the nature of high nonconvexity in minimax, and (2) the computational cost is high. To address these challenges, we propose optimizing our formulation simply by {\em sampling} pairs $(\omega_i, \lambda)$, indeed discretizing the continuous spaces, and {\em selecting} one pair that maximizes the hinge losses. In this way the minimax problem will be converted to a traditional deep learning problem that can be well solved with sufficient training data. We list our training algorithm in Alg. \ref{alg:DLC} for reference. %Fig. \ref{fig:implementation} illustrates our sampling-based training procedure, where we sample the Euler angles for computing rotation matrices (so that the linear combinations make sense in registration) and translation vectors to generate different $\omega$'s for training.

\bfsection{Implementation.}
Eqs. \ref{eqn:con1}, \ref{eqn:con2}, \ref{eqn:con3} define a large pool of inequalities for each data point with varying $\omega_i$ and $\lambda$, where any inequality will return a hinge loss. To address such computational issues, as listed in Alg. \ref{alg:DLC} we propose a {\em sampling} based neural network training algorithm. Specifically,
\begin{itemize}[nosep, leftmargin=*]
    \item {\em Sampling from $\mathcal{N}_{\omega_i^*}$ for $\omega_i$:} Same as SGD, we sample a fixed number of $\omega_i$ for each ground truth, and then compute the average loss over all the samples.

    \item {\em Hyperparameter $\lambda\in[0,1]$:} We simply take $\lambda$ as a predefined hyperparameter that are shared by all the constraints. We have evaluated the way of randomly sampling multiple pairs of $(\lambda, \omega_i)$ and then computing the average hinge loss for learning. We observe that the results are very similar to those with the predefined $\lambda$, but the training time is much longer, as for each pair of $(\lambda, \omega_i)$ we have to compute the hinge losses. 
    
\end{itemize}

\bfsection{Computational Complexity.}
Our training time is approximately twice more than the baseline, given the same number of epochs and the same hardware. This can be deduced from Eqs. \ref{eqn:con1}-\ref{eqn:con3}. At test time, the network inference complexity will be unchanged, because our method does not change the network architecture.

\bfsection{Star-Convexity \vs Convexity.}
Our star-convexification is more like learning one-point convexity \cite{li2017convergence}, where we only impose the (local) convexity at the ground-truth parameters within a local region (controlled by $\Delta$). This way is appreciated by the optimization problems in downstream tasks. In contrast, convexification may be too strong and unnecessary so that every point, including the ground-truth, within a local region has to be convex. This will also lead to more challenges in sampling, thus slowing down the training efficiency and convergence. In practice, we often observe that convexification leads to inferior performance but much higher computational cost, compared with star-convexification.

\begin{algorithm}[t]
    \SetAlgoLined
    \SetKwInOut{Input}{Input}\SetKwInOut{Output}{Output}
    \Input{training data $\{(x_i, \omega^*_i)\}$, deep loss function $h_{\theta}$, hyperparameters $\lambda, \rho, \mu$}
    \Output{network parameters $\theta^*$}
    \BlankLine
    Randomly initialize $\theta$;
            
    \Repeat{Converge or maximum number of iterations is reached}{
        
        Randomly select a training datapoint $(x_i, \omega^*_i)$;
        
        Randomly sample multiple $\omega_i\sim\mathcal{N}_{\omega_i^*}$; %\tcp{Could be multiple}

        Compute $\Tilde{\omega}_i = (1-\lambda)\omega_i^* + \lambda\omega_i$ for each sample of $\omega_i$;
        
        Update $\theta$ by solving Eq. \ref{eqn:problem} with constraints Eqs. \ref{eqn:con1}, \ref{eqn:con2}, \ref{eqn:con3};
    }
    \Return $\theta^* \leftarrow \theta$;
    \caption{Deep Loss Convexification with Star-Convexity}\label{alg:DLC} 
\end{algorithm}
% \vspace{-100mm}

\bfsection{Differences from Our Prior Work \cite{zhang2023prise}.}
There are two distinct changes in our formulation, compared with \cite{zhang2023prise}. They are:
\begin{itemize}[nosep, leftmargin=*]
    \item {\em No adversarial training %(\eg \cite{bai2021recent, qian2022survey}) 
    (\ie a minimax problem as in \cite{zhang2023prise}) is needed.} We observe that without the max selection on the hinge losses, our new test performance is very similar to that of \cite{zhang2023prise}, but the training stability seems to be improved slightly.

    \item {\em An extra constraint in Eq. \ref{eqn:con1} is added.} We find that this new constraint can significantly improve the performance on certain tasks such as point cloud registration, compared with \cite{zhang2023prise}, and so far we have not observed any case where the performance of our new formula is worse than \cite{zhang2023prise}. One hypothesis to explain this is that this new constraint helps smooth the loss landscape to better approximate the star-convexity shape. Please refer to our experimental section for more details.
\end{itemize}

\subsubsection{Near-Optimal Solutions of Star-Convexification}\label{sssec:analysis}

\cite{pmlr-v125-hinder20a} has shown that the GD based algorithms can find near-optimal solutions for the optimization of star-convex functions. Below we analyze the near-optimal property of our star-convexification, assuming that the networks can learn star-convex loss landscapes \wrt\ predictions perfectly.

\begin{lemma}\label{lem:2}
Letting function $h_{\theta^*}$ in Eqs. \ref{eqn:problem}-\ref{eqn:con3} be $L$-Lipschitz and $\mu$-strongly star-convex, %Supposing $L\geq\frac{\gamma\mu}{2}$, 
then it holds that for an arbitrary datapoint $x$, the difference between its prediction and ground truth, $\|\omega-\omega^*\|$, will be upper-bounded by $\frac{2L}{\mu}$. That is,
\begin{align}\label{eqn:lem2}
    \|\omega-\omega^*\|\leq\frac{2L}{\mu}.
\end{align}
\end{lemma}
\begin{proof}
    Based on Eq. \ref{eqn:con2} and the assumptions, we have
    \begin{align}\label{eqn:w_dis}
        \|\omega - \omega^*\|^2 \leq \frac{2}{\mu}\Big[h_{\theta^*}(\omega) - h_{\theta^*}(\omega^*)\Big] \leq \frac{2L}{\mu}\|\omega - \omega^*\|,
    \end{align}
    leading to Eq. \ref{eqn:lem2} when $\|\omega - \omega^*\|>0$ holds. If $\|\omega - \omega^*\|=0$ Eq. \ref{eqn:lem2} still holds, which completes our proof.
\end{proof}
From this lemma we can clearly see that with star-convex constraints the predictions are always within the local regions of the ground truth. In fact, the distance is upper-bounded by the {\em contrastive loss}, which contributes to the hinge losses that are appended to the original loss in our formulation. This observation indicates that the contribution of the hinge losses to a well-trained network may be higher than the original loss. When the star-convex constraint is violated, \ie $\gamma_{\omega}>0$ in Eq. \ref{eqn:con2}, the upper-bound in Eq. \ref{eqn:lem2} will increase to 
$\frac{L}{\mu}\left[1+\left(1+\frac{2\mu}{L^2}\cdot\gamma_{\omega}\right)^{\frac{1}{2}}\right]$. Therefore, when $\gamma_{\omega}$ is small, the upper-bound grows approximately with rate $O(\gamma_{\omega})$, and when $\gamma_{\omega}$ is large, the upper-bound grows with rate $O(\sqrt{\gamma_{\omega}})$. Note that Lemma \ref{lem:2} holds for Eq. \ref{eqn:iteration_test} as well.

\begin{lemma}\label{lem:3}
    Let $\omega_{i,t}$ be the prediction for data $x_i$ at the $t$-th iteration and $\omega_i^*$ be the ground truth. Supposing that the prediction errors are uncorrelated, \ie $\mathbb{E}_i[(\omega_{i,m}-\omega_i^*)^T(\omega_{i,n}-\omega_i^*)]=0, m\neq n\in[T]$ where $\mathbb{E}$ denotes the expectation operator, then based on the prediction in Eq. \ref{eqn:iteration_test} and Lemma~\ref{lem:2}, it holds that
    \begin{align}
        \mathbb{E}_i\left[\left\|\left(\frac{1}{T}\sum_{t=1}^T\omega_{i,t}\right) - \omega_i^*\right\|^2\right] \leq \frac{4L^2}{\mu^2}\cdot\frac{1}{T} = O\left(\frac{1}{T}\right).
    \end{align}
\end{lemma}
\begin{proof}
    By referring to Sec. 14.2 in \cite{bishop2006pattern}, we can show that
    \begin{align}
        \mathbb{E}_i\left[\left\|\left(\frac{1}{T}\sum_{t=1}^T\omega_{i,t}\right) - \omega_i^*\right\|^2\right] & = \frac{1}{T^2}\sum_{t=1}^T\mathbb{E}_i\Big[\|\omega_{i,t} - \omega_i^*\|^2\Big]. \nonumber
    \end{align}
    Substituting Eq. \ref{eqn:lem2} into this equation completes our proof.
\end{proof}
This lemma states that using the average solutions the prediction error will decrease with the number of iterations $T$ growing, on average. This shows the robustness of the prediction with the average solution. Though the uncorrelated assumption is very strong and in our case it seems not to hold, in practice we do observe significant performance improvement over a few iterations. In our experiments, we will tune the number of iterations as a hyperparameter to demonstrate its effect on performance if explicitly mentioning it; otherwise, by default there will be no iterations and such a tuning procedure (see Sec. \ref{ssec:exp_pc} for details).

\section{Experiments}\label{sec:exp}

\begin{figure}[t] 
\centering 
\vspace{-5mm}
\includegraphics[width=.8\linewidth]{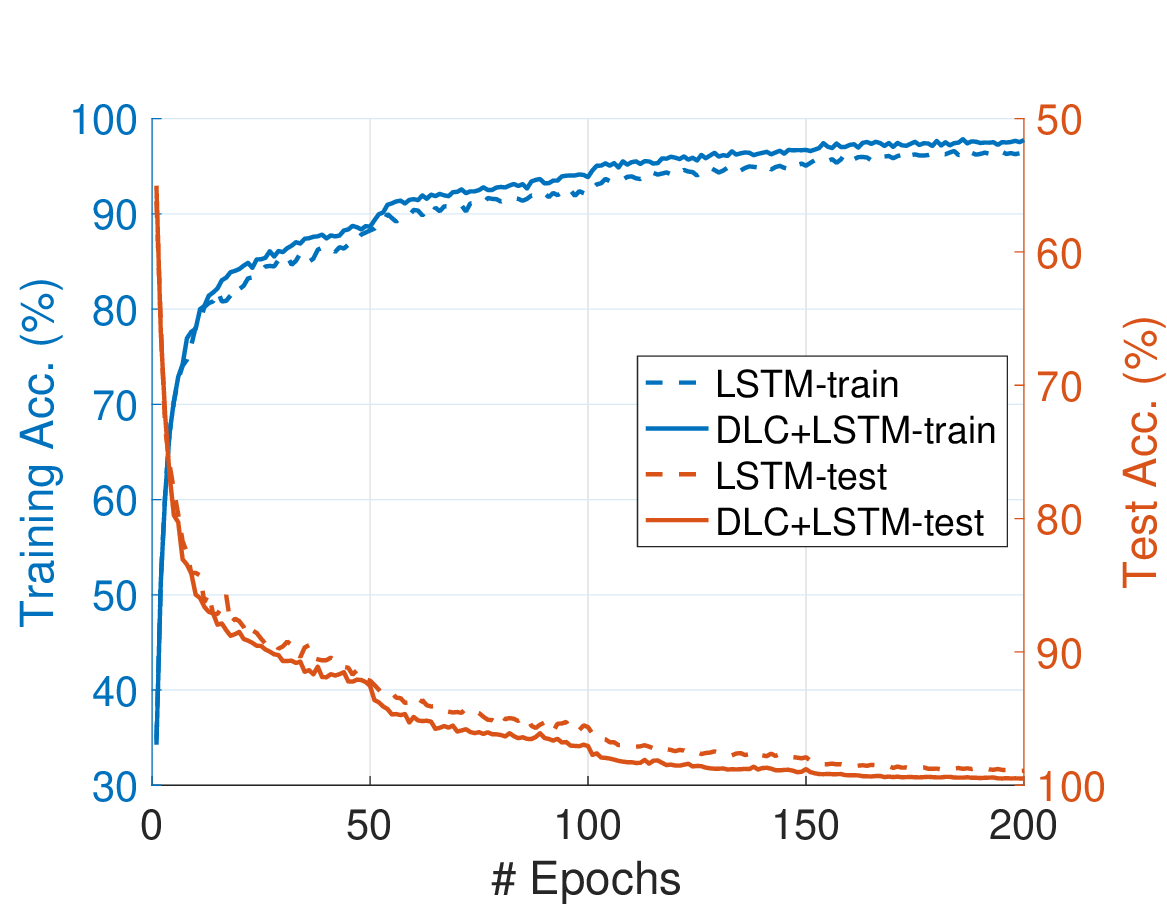}
\vspace{-1mm}
\caption{Performance comparison in both training and testing for LSTM on Pixel-MNIST with and without our DLC.} 
\vspace{-5mm}
\label{fig:lstm}
\end{figure} 

\begin{figure*}[t] 
\centering
\begin{minipage}{0.30\linewidth}
    \includegraphics[width=.8\linewidth]{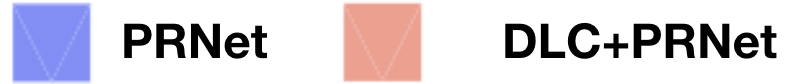}
\end{minipage}

\centering 
\begin{minipage}[b]{0.24\linewidth}
	\begin{center}
	   \centerline{\includegraphics[width=.8\linewidth, keepaspectratio,]{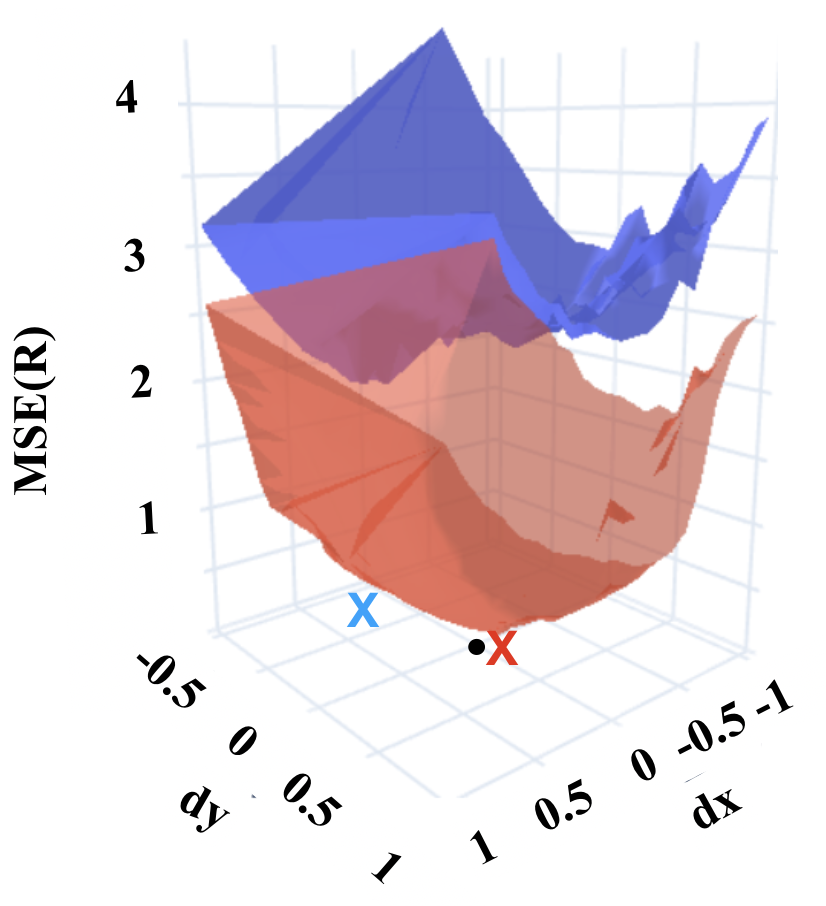}}
		\centerline{(a)}
	\end{center}
\end{minipage}
\hfill
\begin{minipage}[b]{0.24\linewidth}
	\begin{center}
	   \centerline{\includegraphics[width=.8\linewidth, keepaspectratio,]{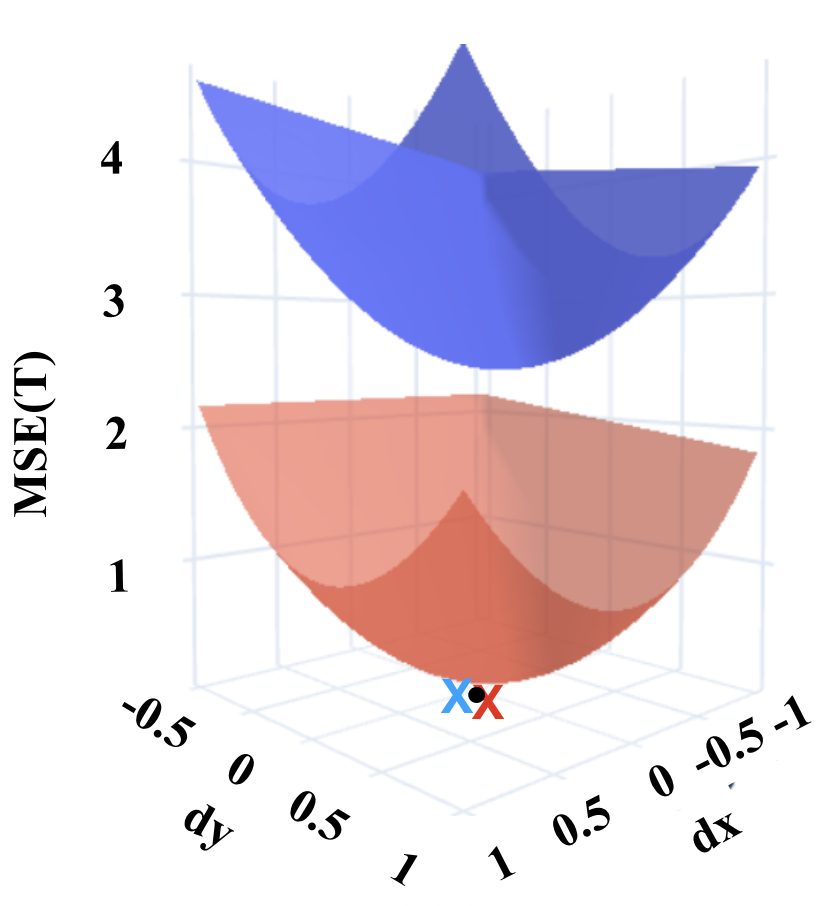}}
		\centerline{(b)}
	\end{center}
\end{minipage}
\hfill
\vline height 150pt depth -25pt width 1pt
\hfill
\begin{minipage}[b]{0.24\linewidth}
	\begin{center}
	   \centerline{\includegraphics[width=.8\linewidth, keepaspectratio,]{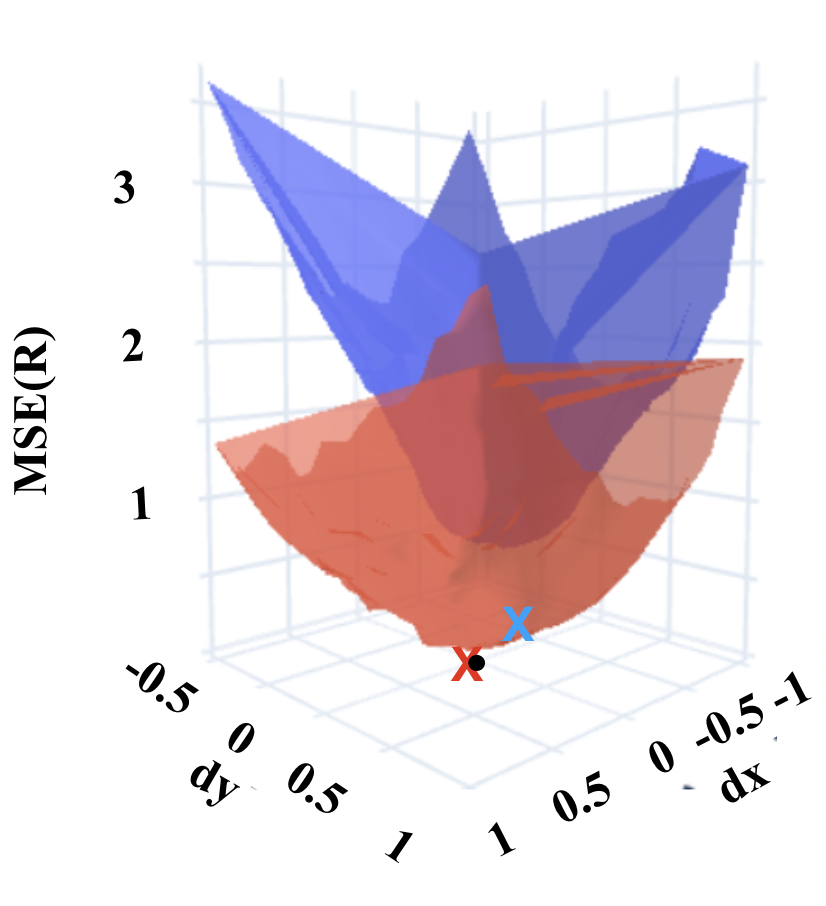}}
		\centerline{(c)}
	\end{center}
\end{minipage}
\hfill
\begin{minipage}[b]{0.24\linewidth}
	\begin{center}
	   \centerline{\includegraphics[width=.8\linewidth, keepaspectratio,]{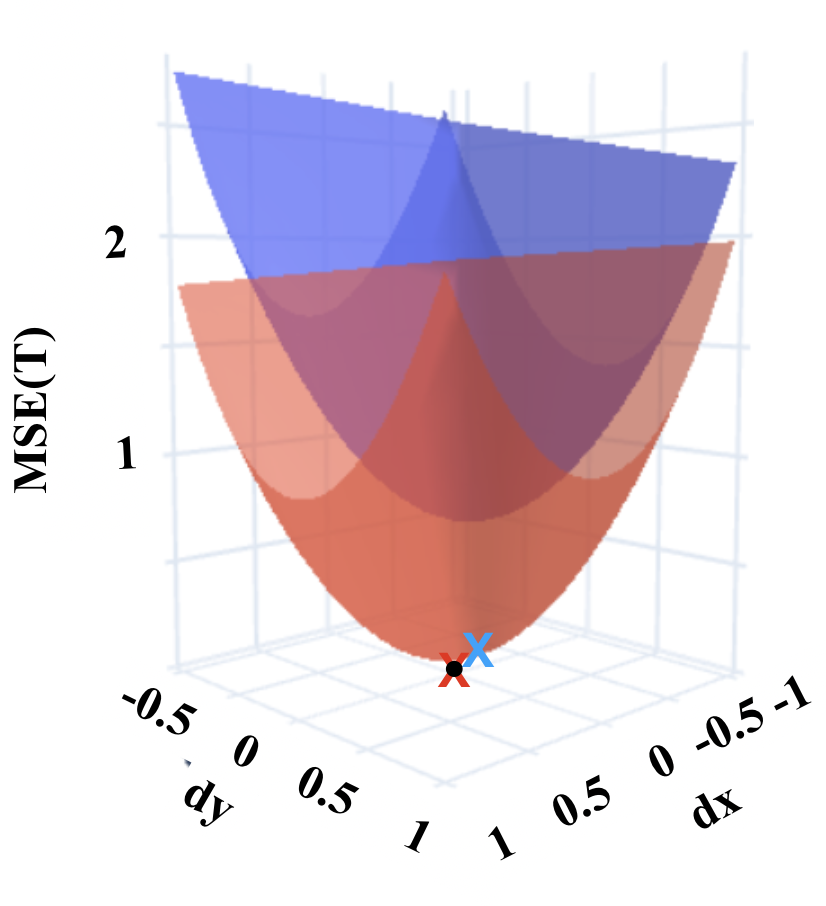}}
		\centerline{(d)}
	\end{center}
\end{minipage}
\vspace{-5mm}
\caption{Test-time loss landscape comparison using PRNet on ModelNet40: {\bf (a,c)} rotation matrices, and {\bf (b,d)} translations.} 
% \vspace{-1mm}
\label{fig:iterative}
\end{figure*} 

\begin{figure*}[t] 
\centering 
\vspace{-3mm}
\includegraphics[width=.8\linewidth]{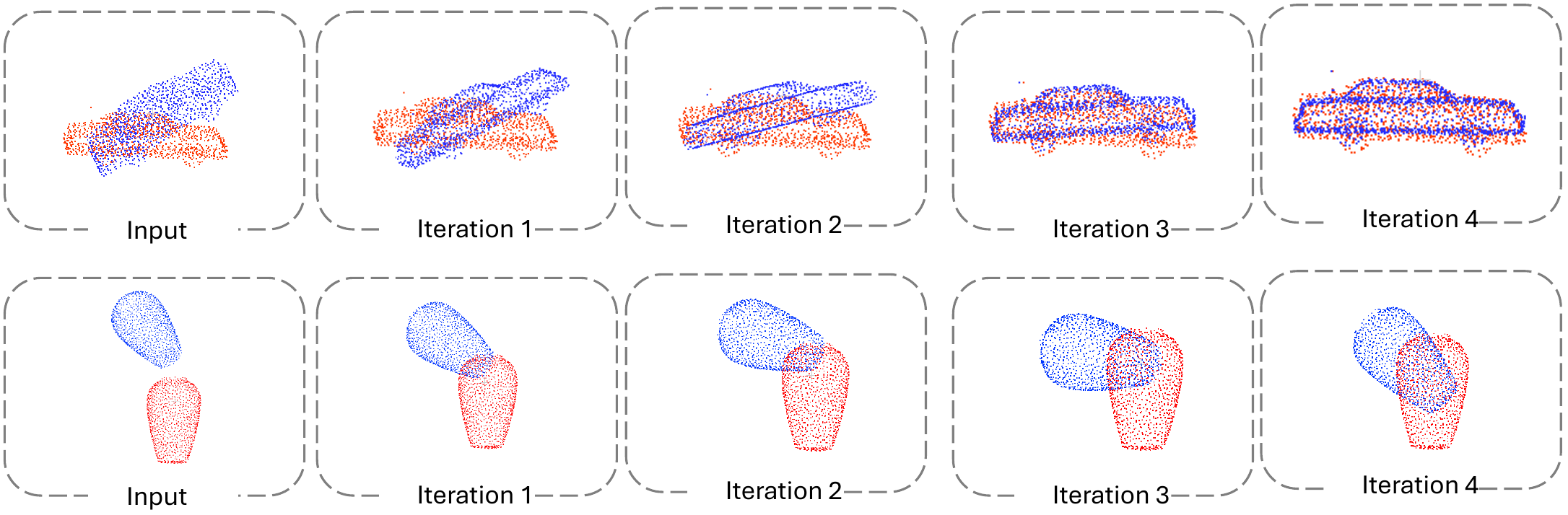}
\vspace{-2mm}
\caption{Illustration of test-time point cloud registration for {\bf (top)} a successful case and {\bf (bottom)} a failure case using DLC+PRNet.}
\vspace{-4mm}
\label{fig:iteration_demo}
\end{figure*} 

\begin{figure}[t] 
\centering 
\vspace{-3mm}
\includegraphics[width=.8\linewidth]{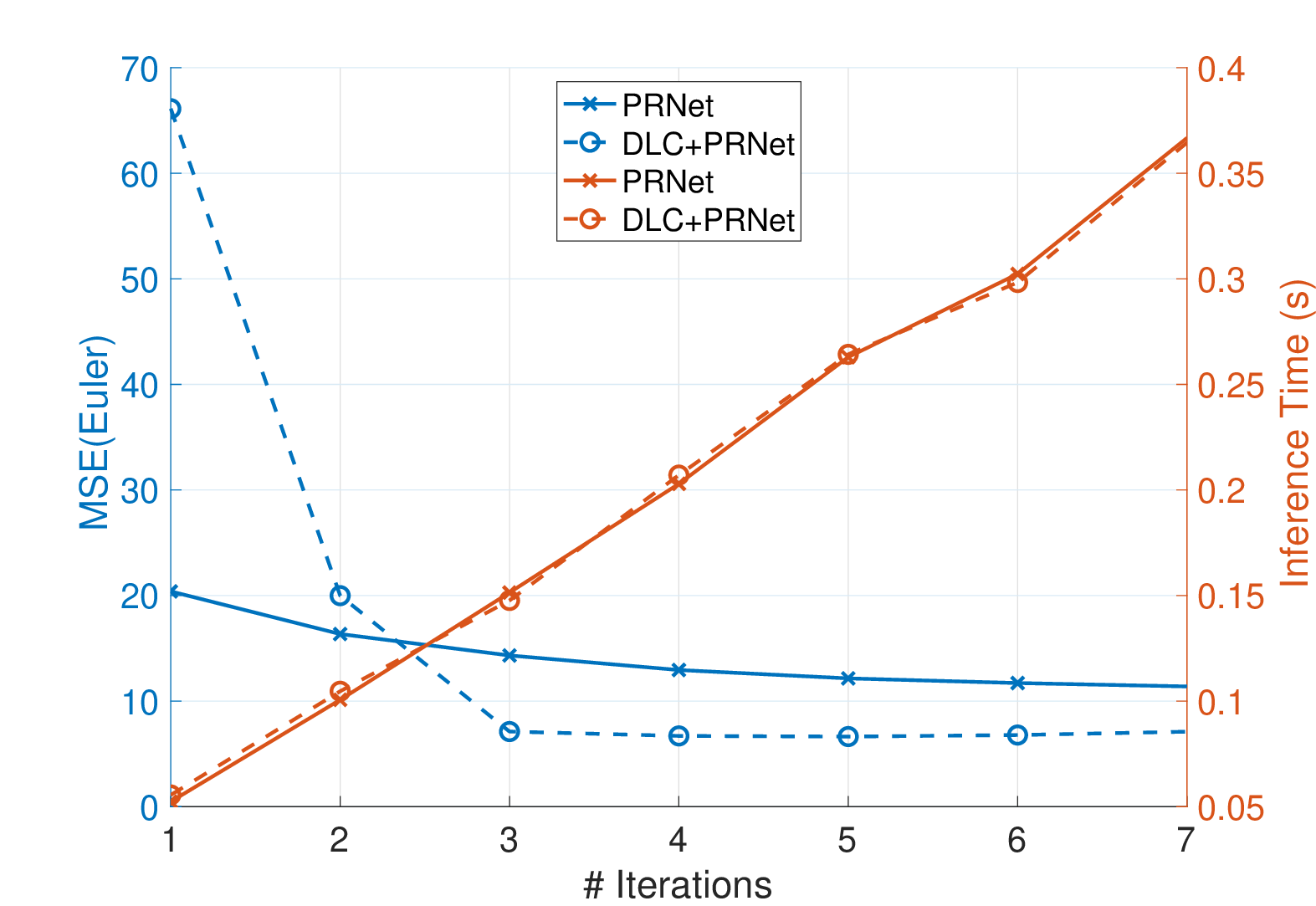}
\vspace{-2mm}
\caption{Illustration of test-time performance comparison with and without DLC on ModelNet40 using PRNet.} 
\vspace{-3mm}
\label{fig:iteration}
\end{figure} 

\begin{figure}[t] 
\centering 
% \vspace{-3mm}
\includegraphics[width=.8\linewidth]{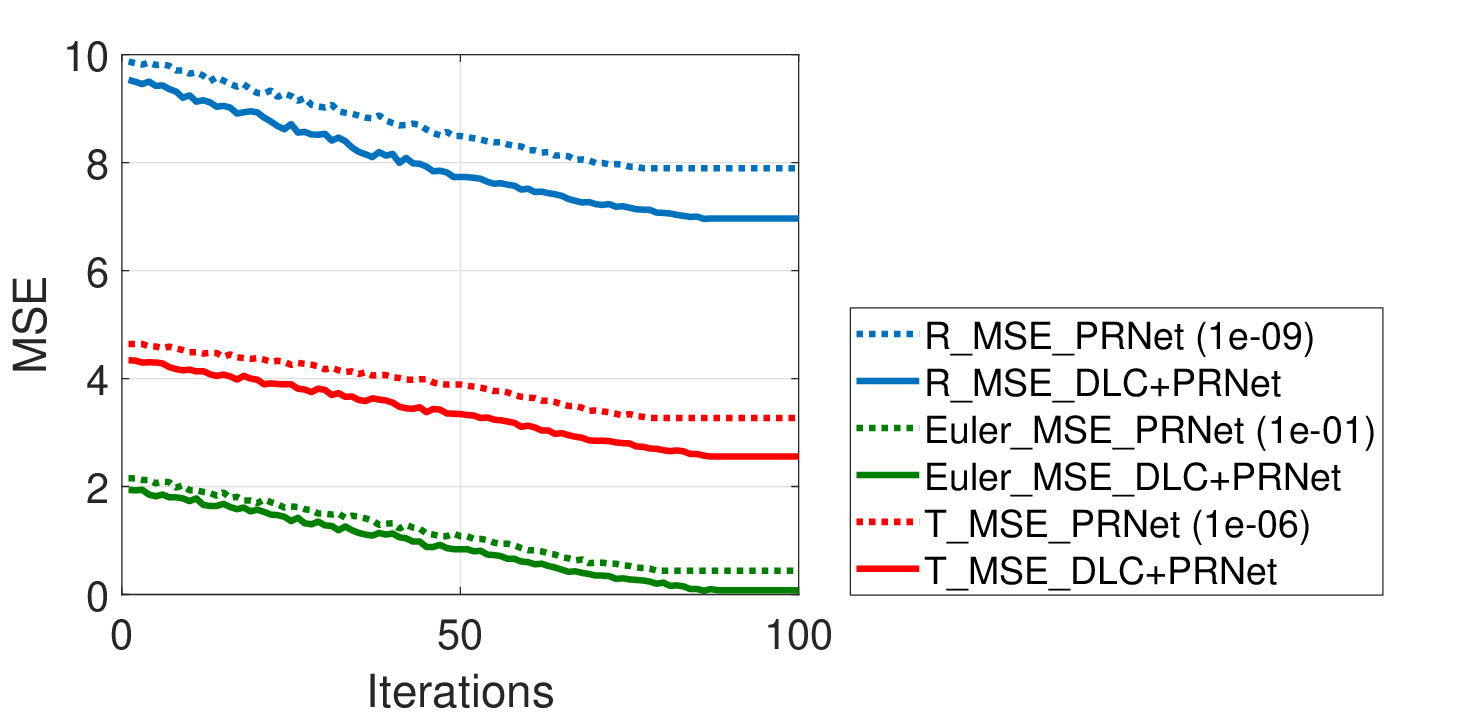}
\vspace{-1mm}
\caption{Comparison results for PRNet \vs DLC+PRNet.} 
\vspace{-5mm}
\label{fig:kitti}
\end{figure}

\subsection{Recurrent Neural Networks: A Naive Iterative Model}

\setlength{\intextsep}{0pt}
\setlength{\columnsep}{5pt}
\begin{wraptable}{r}{.45\linewidth}%\vspace{-1mm}
    \small
    \caption{RNN classification results on Pixel-MNIST.}\label{tab:lstm}
    \vspace{-2mm}
    \setlength{\tabcolsep}{1pt}
        \begin{tabular}{cc}
    		\toprule  
Method & Acc. (\%) \\
\hline
SMPConv \cite{kim2023smpconv} & 99.75 \\  
S4 \cite{gu2021efficiently} & 99.63 \\
FlexTCN-6 \cite{romero2021flexconv} & 99.62 \\
LSSL \cite{gu2021combining} & 99.53 \\
{\bf Ours: DLC+LSTM} & \underline{{\bf 99.50}} \\
LEM \cite{rusch2021long} & 99.50 \\
Deep IndRNN \cite{li2019deep} & 99.48 \\
coRNN \cite{konstantin2020coupled} & 99.40 \\
LipschitzRNN \cite{erichson2020lipschitz} & 99.40 \\
STAR \cite{turkoglu2021gating} & 99.40 \\
CKCNN \cite{romero2021ckconv} & 99.32 \\
\hline
BN LSTM \cite{cooijmans2016recurrent} & 99.00 \\
LSTM \cite{arjovsky2016unitary} & 98.2 \\

\bottomrule
	    \end{tabular}    
     % \vspace{-3mm}
\end{wraptable}

Recurrent neural networks (RNNs) are a family of well-known classic iterative models to process sequential data. To demonstrate our performance as a toy example, we will evaluate DLC based on long short-term memory (LSTM) \cite{hochreiter1997long}, a popular RNN in practice, for image classification.

\bfsection{Dataset.}
We utilize a benchmark dataset, Pixel-MNIST, \ie pixel-by-pixel sequences of images in MNIST \cite{lecun-mnisthandwrittendigit-2010} where each 28x28 image is flattened into a 784 time-step sequence vector and normalized as zero mean and unit variance. 

\bfsection{Training \& Testing Protocols.}
We modify the LSTM code\footnote{\url{https://github.com/pytorch/examples/blob/main/mnist_rnn/main.py}} to integrate it with our training algorithm in Alg. \ref{alg:DLC}, and then train the LSTM using SGD with fine-tuned learning rate but no scheduler. We use the grid search to determine $\rho, \lambda, \mu$. To generate $\omega_i$, we add Gaussian noise to each entry in the ground-truth one-hot vector that is randomly sampled from a distribution $\mathcal{N}(0,1)$, and then apply softmax to generate a probability vector. At test time, the sequential data is fed to the learned LSTM, same as the conventional RNNs. It is worth of emphasizing that the LSTM architecture is unchanged at all, and thus the performance of our learned model is fair to compare with the orginal LSTM.

\bfsection{Results.}
We illustrate the performance with and without DLC in Fig. \ref{fig:lstm} for both training and testing. Clearly, our DLC can consistently improve the LSTM performance in both cases. We further list our comparison results in Table~\ref{tab:lstm}, where all the comparative numbers are cited from the leaderboard at the popular website\footnote{\url{https://paperswithcode.com/sota/sequential-image-classification-on-sequential?p=full-capacity-unitary-recurrent-neural}} that summarizes 22 methods ranging from 1998 to 2023. Table \ref{tab:lstm} lists top-10 methods as well as two LSTM models, where our result ranks at the 5th place. Note that the top-4 models are all much more complicated than LSTM with much more parameters. Without any modification of the LSTM architecture, just learning better weights with DLC we can easily outperform the rest models that are often more complicated as well. All these observations demonstrate that the convex-like loss landscapes learned with DLC help converge to near-optimal solutions.

% % \setlength{\intextsep}{0pt}
% \begin{table}[t]
% \caption{Classification performance on Pixel-MNIST.}\label{tab:lstm}
% % \setlength{\tabcolsep}{1.5pt}
% \centering
% \begin{tabular}{cc}
% \toprule  
% Method & Acc. (\%) \\
% \hline
% Vanilla RNN & 94.10 \\  
% FastRNN \cite{kusupati2018fastgrnn} & 96.44 \\
% ShaRNN \cite{dennis2019shallow} & 97.87 \\
% AntisymmetricRNN \cite{chang2019antisymmetricrnn} & 98.01 \\
% iRNN \cite{Kag2020RNNs} & 98.13 \\
% FastGRNN \cite{kusupati2018fastgrnn} & 98.72 \\
% SBO-RNN \cite{zhang2021sbo} & 98.85 \\
% IndRNN \cite{li2018independently} & 99.00 \\
% LipschitzRNN \cite{erichson2020lipschitz} & 99.00 \\
% MomentumRNN \cite{nguyen2020momentumrnn} & 99.08 \\
% \hline
% LSTM & 98.91 \\
% {\bf DLC+LSTM} & {\bf 99.50} \\
% \bottomrule
% \end{tabular}
% \end{table} 

\begin{table*}[t]
\centering\small
\caption{Performance comparison on {\em unseen point classes} in ModelNet40 and 3DMatch. ModelNet40 uses the pre-generated dataset. 3DMatch uses the same data processing as ModelNet40.}
\vspace{-2mm}
\setlength{\tabcolsep}{2.8pt}{
\begin{tabular}{c|c|cccc|cccc}
\toprule
& & \multicolumn{4}{c|}{w/o ICP refinement}  & \multicolumn{4}{c}{with ICP refinement}\\
& & DCP & \textbf{DLC+DCP} &PRNet & \textbf{DLC+PRNet} & DCP & \textbf{DLC+DCP} & PRNet & \textbf{DLC+PRNet}\\ 
\hline

\multirow{3}{*}{ModelNet40} & MSE(R)  & 0.0013 & \textbf{0.0002} & 0.0038 & 0.0019  & 0.0001 & \textbf{2.99e-5} & 0.0029 & 0.0016 \\

% \multirow{6}{*}{ModelNet40} & MSE(R)  & \pm 0.0034 & \pm \textbf{0.0007} & \pm 0.0232 & \pm 0.0223 & \pm 0.0017 & \pm \textbf{0.0007} & \pm 0.0237 & \pm 0.0205 \\

% \multirow{3}{*}{ModelNet40} & MSE(R)  & 0.0013 \pm 0.0034 & \textbf{0.0002} \pm 0.0007 & 0.0038 \pm 0.0232 & 0.0019 \pm 0.0223 & 0.0001 \pm 0.0017 & \textbf{2.99e-5} \pm 0.0007 & 0.0029 \pm 0.0237 & 0.0016 \pm 0.0205 \\

& MSE(Euler)  & 6.8747 & \textbf{1.0836} & 20.3915 & 11.3448  & 0.7928 & \textbf{0.1705} & 16.5405 & 9.5568 \\

% & MSE(Euler)  & \pm 17.6783 & \pm \textbf{3.778} & \pm 135.8661 & \pm  146.2587 & \pm 9.5338 & \pm \textbf{3.9184} & \pm 152.0734 & \pm 136.9585 \\

% & MSE(Euler)  & 6.8747 \pm 17.6783 & \textbf{1.0836} \pm 3.778 & 20.3915 \pm 135.8661 & 11.3448 \pm  146.2587 & 0.7928 \pm 9.5338 & \textbf{0.1705} \pm 3.9184 & 16.5405 \pm 152.0734 & 9.5568 \pm 136.9585 \\

& MSE(T)  & \textbf{1.79e-5} & 2.55e-5 & 0.0001 & 0.0002 & \textbf{2.72e-7} & 1.63e-6 & 1.60e-5 & 8.10e-6 \\

% & MSE(T)  & \pm \textbf{\text{1.79e-5}} & \pm \text{5.38e-5} & \pm 0.0004 & \pm 0.0003 & \pm \textbf{5.87e-6} & \pm \text{2.11e-5} & \pm 0.0003 & \pm \text{9.89e-5} \\

% & MSE(T)  & \textbf{1.79e-5} \pm 1.79E-05 & 2.55e-5 \pm 5.38E-05 & 0.0001 \pm 0.0004 & 0.0002 \pm 0.0003 & \textbf{2.72e-7} \pm 5.87E-06 & 1.63e-6 \pm 2.11E-05 & 1.60e-5 \pm 0.0003 & 8.10e-6 \pm 9.89E-05 \\

%%% 3dmatch ###

\hline
\multirow{3}{*}{3DMatch} & MSE(R) & - & - & 0.1112 & {\bf 0.0560} & - & - & 0.0299 & {\bf 0.0006} \\
% & MSE(R) & - & - & \pm 0.2259 & \pm \textbf{0.0735} & - & - & \pm 0.1308 & \pm {\bf 0.0093} \\
& MSE(Euler) & - & - & 1104.9437 & {\bf 294.2320} & - & - & 259.4553 & {\bf 3.1288} \\
% & MSE(Euler) & - & - & \pm 2866.6394 & \pm \textbf{521.6077} & - & - & \pm 1371.2459 & \pm \textbf{50.3634} \\
& MSE(T) & - & - & 1.6689 & {\bf 0.3372} & - & - & 0.2076 & {\bf 0.0021} \\
\bottomrule
\end{tabular}}
\label{table:reliability}
\vspace{-3mm}
\end{table*}

\begin{table}[t]
\centering\small
\caption{Tranferability results on ShapeNetCore, where both PRNet and DLC+PRNet are pretrained on ModelNet40.}
\vspace{-2mm}
\setlength{\tabcolsep}{1.7pt}{
\begin{tabular}{c|cc|cc}
\toprule
% \hline
& \multicolumn{2}{c|}{w/o ICP refinement}  & \multicolumn{2}{c}{with ICP refinement}\\
 & PRNet & \textbf{DLC+PRNet} & PRNet & \textbf{DLC+PRNet} \\ 
\hline
\multicolumn{5}{c}{Unseen Class 14 (\# of objects=662)} \\
MSE(R) & 0.0004 & {\bf 0.0002} & 0.0001 & {\bf 8.0e-8}   \\
% MSE(R) & \pm  0.0022 & \pm {\bf 0.0013} & \pm 0.0013 & \pm \textbf{ 1.2e-6}  \\
MSE(Euler) &  2.0465 & {\bf 1.0666} & 0.4404 & {\bf 0.0003} \\
% MSE(Euler) & \pm 12.4638 & \pm{\bf 11.4734} & \pm6.6732 & \pm{\bf 0.0046} \\
MSE(T) & {\bf 0.0001} & 0.0004 & {\bf 7.0e-6} & 1.0e-5 \\
\hline
\multicolumn{5}{c}{All the Classes (\# of objects=4071)} \\
MSE(R) & 0.0026 & {\bf 0.0025} & 0.0026 & {\bf 0.0019}   \\
% MSE(R) & \pm {\bf 0.0133} & \pm 0.0225 & \pm 0.0203 & \pm {\bf 0.0163}  \\
MSE(Euler) & {\bf 13.6229} & 15.3662 & 14.8344 & {\bf 10.6357} \\
% MSE(Euler) & \pm{\bf 71.704} & \pm183.3369 & \pm131.2069 & \pm{\bf 95.7864} \\
MSE(T) & {\bf 0.0007} & 0.0015 & {\bf 0.0005} & 0.0006 \\
% MSE(T) & \pm{\bf 0.0064} & \pm0.0232 & \pm{\bf 0.0072} & \pm0.0122 \\
\bottomrule
\end{tabular}}
\label{table:transferability}
\end{table}

\subsection{3D Point Cloud Registration}\label{ssec:exp_pc}

\bfsection{Datasets.}
We conduct our experiments on ModelNet40 \cite{wu20153d}, ShapeNetCore \cite{chang2015shapenet} and 3DMatch \cite{zeng20173dmatch} with default training, validation and testing splits. Also, we evaluate our approach on the KITTI odometry dataset that is composed of 11 sequences (00-10) with ground truth poses. Following the settings in \cite{choy2019fully,choy2020deep}, we use sequences 00-05 for training, 06-07 for validation, and 08-10 for testing. Additionally, the ground truth poses of KITTI are refined with ICP as done in \cite{choy2019fully,bai2020d3feat}. 
% We conduct our experiments on ModelNet40 \cite{wu20153d} an ShapeNetCore \cite{chang2015shapenet} with default splits for training,  (if necessary), and testing.

\bfsection{Backbone Networks.}
We adopt two neural network based iterative models, namely deep closest point\footnote{\url{https://github.com/WangYueFt/dcp}} (DCP) and PRNet\footnote{\url{https://github.com/YadiraF/PRNet}}, to demonstrate the benefit of learning with DLC. For our DLC+DCP, we actually concatenate four copies of the DCP network with the same weights as our backbone, similar to the PRNet architecture.

\bfsection{Training \& Testing Protocols.}
Following the setting for each dataset in both training and testing, we select 1024 points for each object. %For training, we only use clean data. For testing, we conduct various augmentations, as listed in Table \ref{table:augmentation method}, only on the test sets to show the robustness of our approach. Note that by default we will test models on clean data without explicitly mentioning it. %We report our results with randomization wherever needed such as parameter initialization running on an Nvidia RTX A5000 GPU server. Randomization is mainly for training and augmentation generation. For evaluation, we found that on a fixed test partition, no variance is detected for any given network during repeated trials under the default precision of Pytorch tensor.
% \begin{wrapfigure}{r}{.44\linewidth}
% \centering 
% \vspace{-5mm}
% \includegraphics[width=1\linewidth]{figures/loss_landscape.png}
% \vspace{-5mm}
% \caption{Loss landscape comparison, where two random entries in the ground truth rotation matrix of a pair of test point clouds randomly selected from ModelNet40 are manipulated while the other entries are fixed. %For fair comparison, we train and test the PRNet under the same setting as the RIR-PRNet.
% }
% \vspace{-3mm}
% \label{fig:loss_landscape}
% \end{wrapfigure} 
We use mean squared error of rotation matrix, MSE(R), mean square error of Euler angle, MSE(Euler), and mean squared error of translation, MSE(T), as our metrics\footnote{\url{https://github.com/vinits5/learning3d}} to evaluate performance. We report average results over three trials with randomly initialized network weights on an Nvidia RTX6000 GPU server.

\bfsection{Hyperparameters.}
%We set the batch size to 32 for DCP and 36 for PRNet, the same numbers as the pretrained models. 
We employ Adam \cite{kingma2014adam} as our optimizer with a learning rate 0.001 and weight decay 0.0001 to match the default settings of our backbones. To simplify our sampling in training, we further reparametrize $\lambda$ as a sigmoid function and $\mu$ as an exponential function with trainable parameters. With grid search, by default we set the number of noisy sampling to 3, maximum test-time iterations to 5, $\rho=0.6$ for DLC+DCP on ModelNet40, and $\rho=1$ for DLC+PRNet on ModelNet40 and 3DMatch, without explicitly mentioning. 

\bfsection{Sampling $\omega$.} 
For point cloud registration, each $\omega$ consists of a rotation matrix and a translation vector. We sample the Euler angles for computing rotation matrices because the linear combinations in the angle space guarantee that the corresponding rotation matrices make sense in registration. Same as before, we sample normal Gaussian noise and add it to each Euler angle and each dimension in translation vectors to generate different $\omega$'s for training.

\subsubsection{Loss Landscape Comparison}
We show the learned loss landscapes in Fig. \ref{fig:iterative} by randomly selecting two pairs of point clouds for registration. For visualization, we take two random dimensions in the rotation matrix (or translation vector) as dx and dy, while fixing the rest dimensions, and then manipulate dx and dy by varying their values where $(0,0)$ denotes the ground truth and $\times$'s denote the local minima. As we see, 
\begin{itemize}[nosep, leftmargin=*]
    \item Our approach can lead to much lower errors around the ground truth, verifying \cite{li2018visualizing} that good performance has a strong correlation with the convexity and smoothness of the loss landscape.

    \item Our loss landscapes are much closer to convex-like shapes centered at the ground truth, as we expect because the ground truth should be a (local) minimum around the ground truth.
\end{itemize}
We also observe similar patterns for DCP and DLC+DCP.

\begin{table}[t]
\centering
\caption{State-of-the-art comparison on {\em seen classes} in ModelNet40 with clean data, where colors red, blue, and green in order indicate the top-3 results in each column. }
\vspace{-2mm}
\small
\setlength{\tabcolsep}{2pt}{
\begin{tabular}{c|cc}
\toprule
Method& MSE(Euler) & MSE(T) \\
\midrule
ICP \cite{3dshapeicp} & 894.8973 & 0.0846 \\
Go-ICP \cite{yang2015go} & 140.4773 & 0.0006 \\
FGR \cite{fgr} & 87.6614 & 0.0001 \\
PointNetLK \cite{yaoki2019pointnetlk} & 227.8703 & 0.0005 \\
DeepGMR \cite{yuan2020deepgmr} & 7.9106 & 0.0001 \\

SPA \cite{spa} & 354.57 & 2.00e-5 \\
VCR-Net \cite{wei2020end} & \textcolor{green}{0.2196} & 0.0011 \\
Graphite \cite{saleh2020graphite} & 7.4400 & 0.3100 \\
Graphite + ICP \cite{saleh2020graphite}& 0.7500 & 0.0900 \\
iPCRNet \cite{sarode2019pcrnet} & 9.5210 & 0.7930 \\
CorsNet \cite{kurobe2020corsnet} & 263.7376 & 0.0001 \\
CPDNet \cite{wang2019coherent} & 68.7241 & 0.0024 \\
FPT \cite{biomed_reg} & 25.1001 & 0.0002 \\
RAR \cite{rarnet} & \textcolor{blue}{0.3715} & 7.00e-5\\
FIRE-Net \cite{wu2021feature} & 0.9025 & 3.60e-5 \\
SDFReg \cite{zhang2023sdfreg} & 73.1709 & 2.30e-5 \\   
LC+PRNet \cite{liu2023self} & 1.6718 & 1.96e-4 \\
LC+DCP \cite{liu2023self} & 15.6025 & 4.00e-4 \\
% AAA-ICP& 0.0217& 217.3592& 0.3002\\
% FastICP& 0.0134& 85.8268& 0.3015\\
% RobustICP& 0.0167& 115.3131& 0.3011\\
% PointNetLK& 0.0820& 435.2127& \textcolor{red}{0.0356}\\
RPM-Net \cite{yew2020rpm} & 38.3098 & 0.0170 \\
% DeepGMR& \textcolor{red}{0.0014} &22.1749&0.8754 \\
\midrule
PRNet \cite{wang2019prnet} & 20.3915 & 0.0001\\
\textbf{DLC+PRNet} & 11.3448 & 0.0002 \\
\textbf{DLC+PRNet+ICP} & 9.5568	& \textcolor{blue}{8.10e-6} \\
% DCP&0.0092&48.2718&0.0830\\
% PRNet& \textcolor{green}{0.0025}& \textcolor{blue}{13.3684} & \textcolor{blue}{0.0827}\\
\midrule
% DCPv1 \cite{wang2019deep} & 6.4806 & \textcolor{green}{3.00e-6}\\
DCP \cite{wang2019deep} & 1.3073 & \textcolor{green}{3.00e-6}\\
\textbf{DLC+DCP} & 1.0836 & 2.00e-5 \\
\textbf{DLC+DCP+ICP} & \textcolor{red}{\bf 0.1705} & \textcolor{red}{\bf 1.60e-6} \\

\bottomrule
\end{tabular}}
\vspace{-3mm}
\label{table:modelnet-clean}
\end{table}

\begin{figure*}[t]
    % \hfill
	\begin{minipage}[b]{0.325\textwidth}
		\centering
			\centerline{\includegraphics[width=\linewidth, keepaspectratio,]{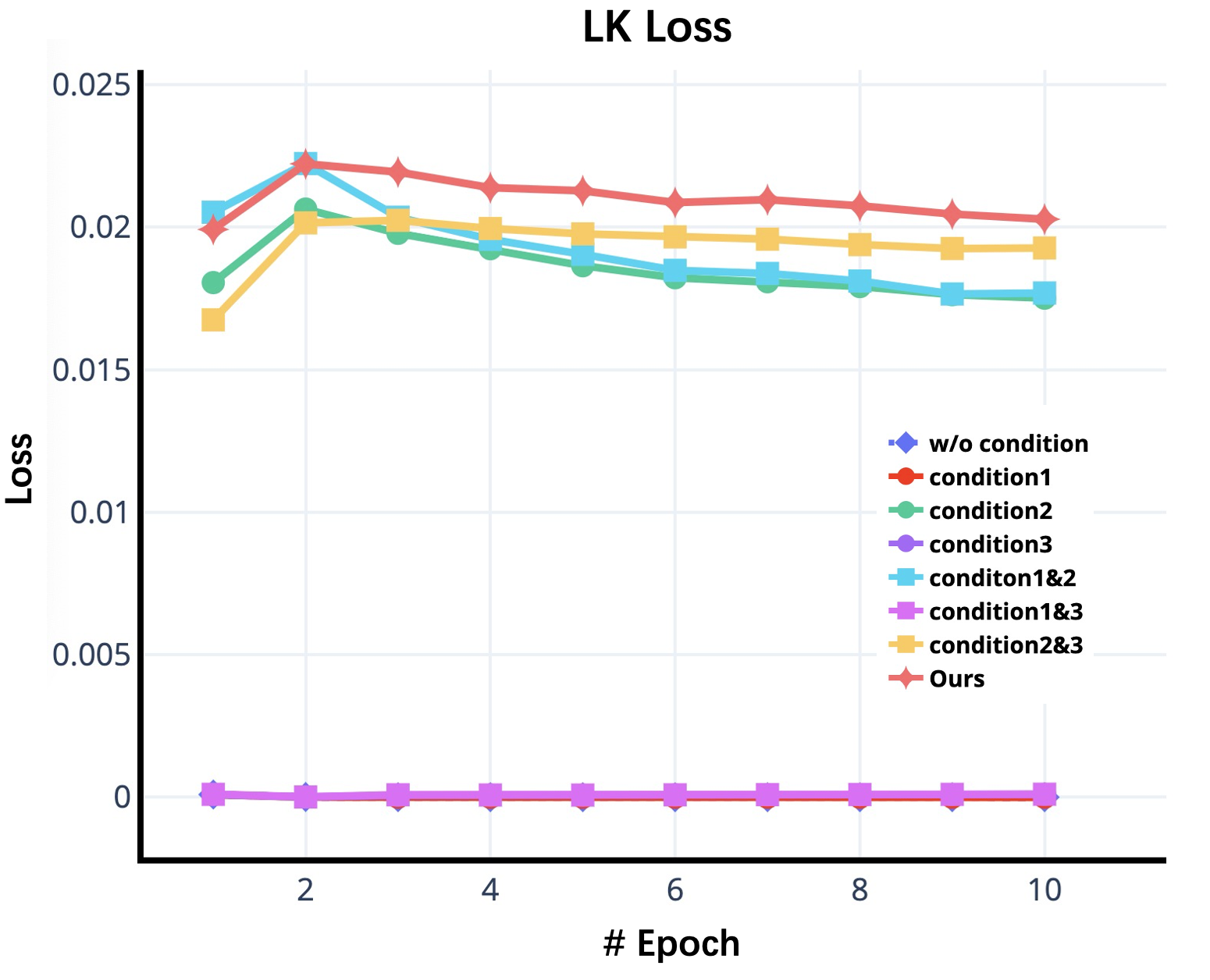}}\vspace{-1mm}
            \centerline{LK loss}	
		
	\end{minipage}
	\hfill
	\begin{minipage}[b]{0.325\textwidth}
		\centering
			\centerline{\includegraphics[width=\linewidth,keepaspectratio]{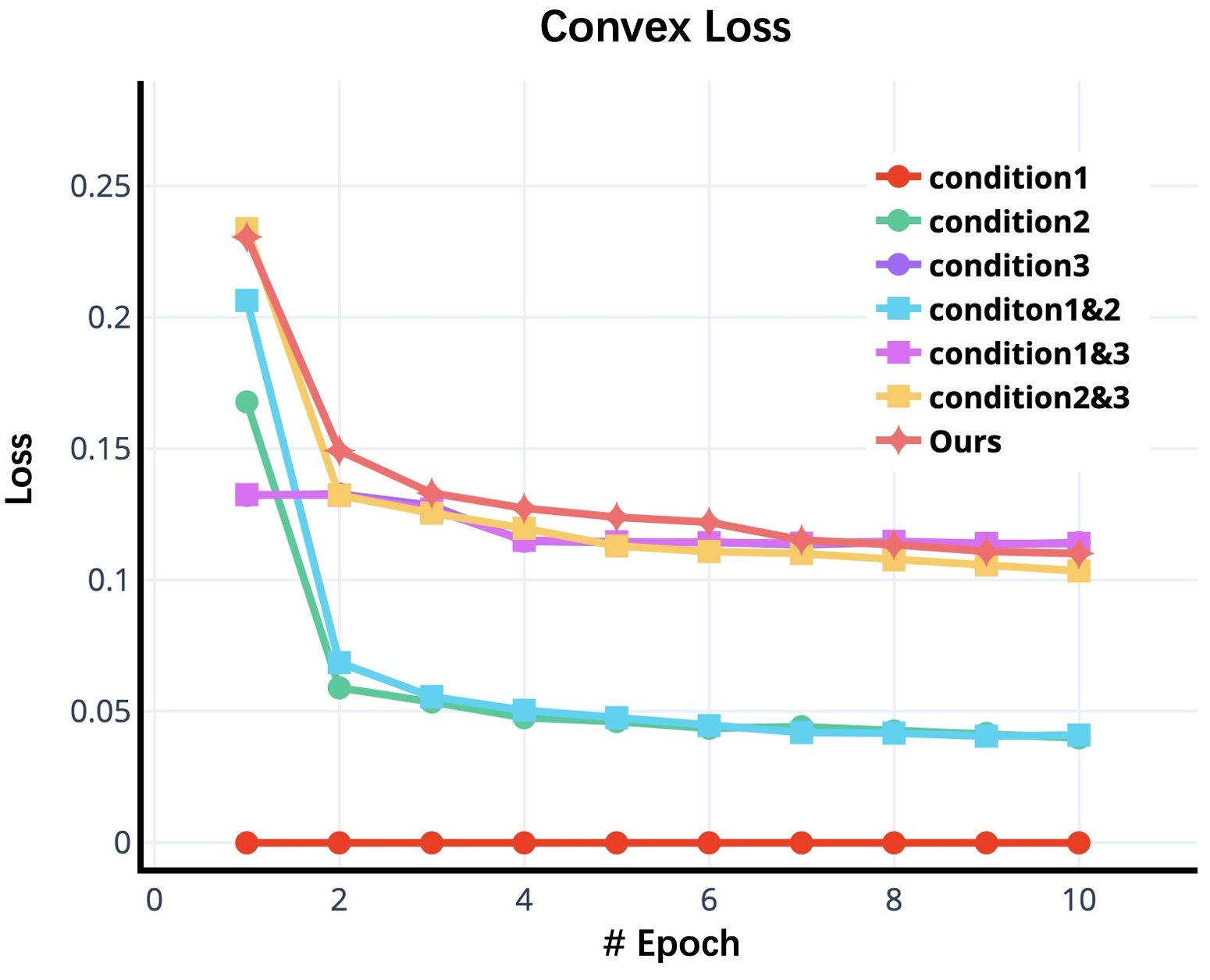}}\vspace{-1mm}
            \centerline{Convexity loss}		
	\end{minipage}
	\hfill
	\begin{minipage}[b]{0.325\textwidth}
		\centering
			\centerline{\includegraphics[width=\linewidth,keepaspectratio]{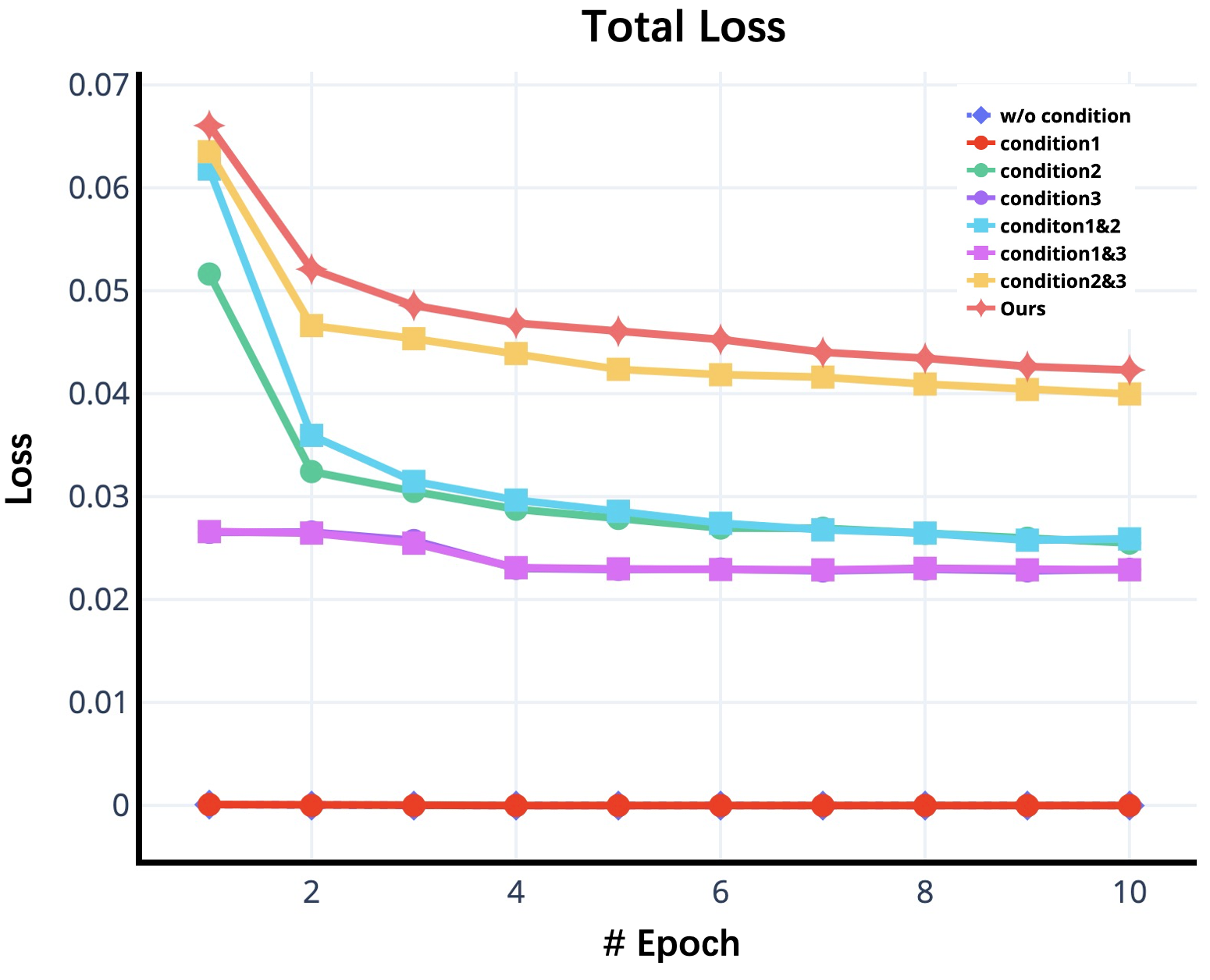}}
            \vspace{-1mm}
            \centerline{Total loss}
		
	\end{minipage}
	\vspace{-2mm}
    \caption{Training loss comparison at stage 3 on Google Earth. Conditions 1-3 correspond to the constraints in Eqs. \ref{eqn:con1}-\ref{eqn:con3} in order.}
\label{fig:loss}
	\vspace{-5mm}
\end{figure*}

\begin{figure}[t] 
\centering 
\vspace{-3mm}
\includegraphics[width=1\linewidth]{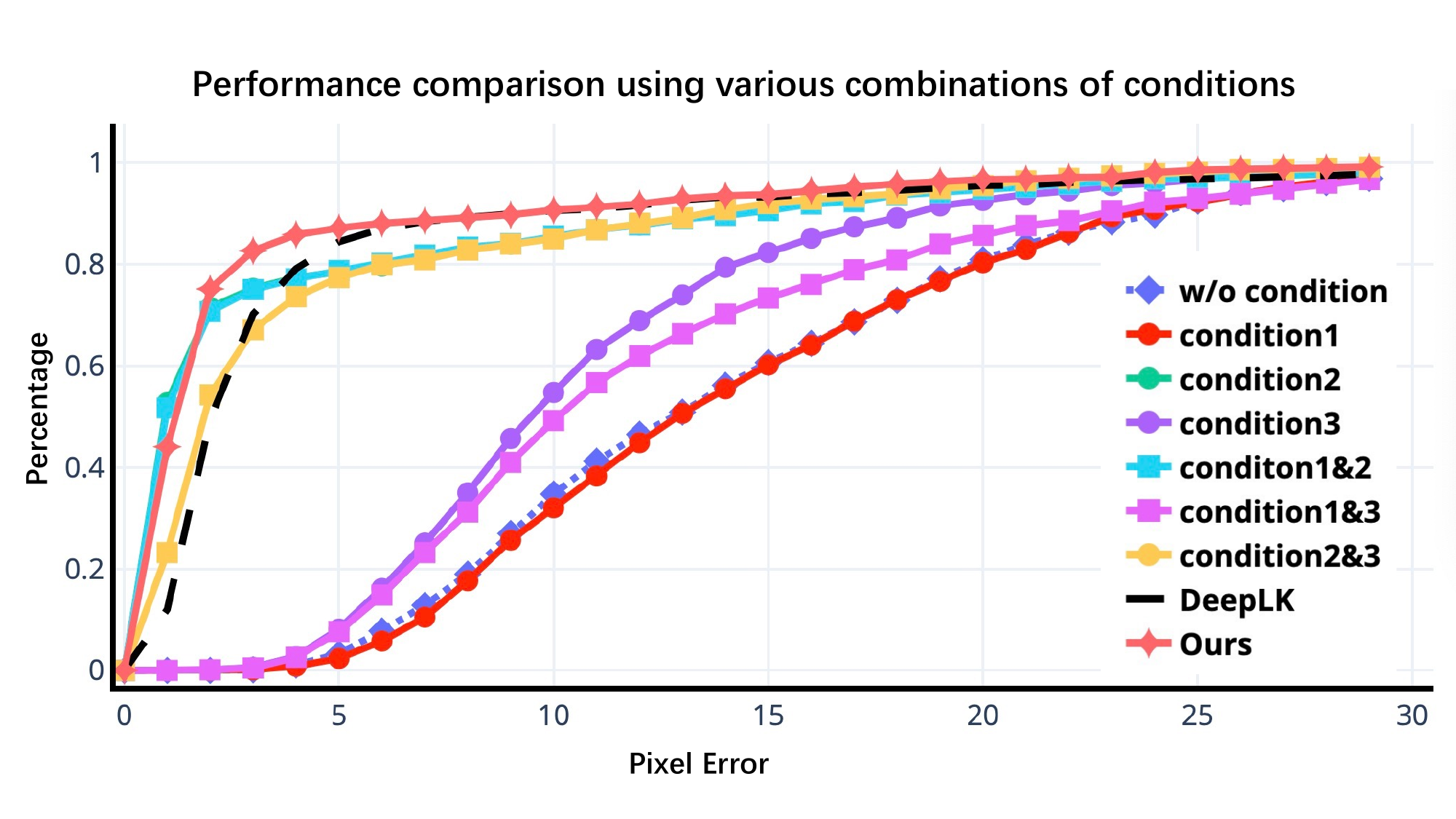}
\vspace{-8mm}
\caption{Performance comparison on Google Earth using the default hyperparameters $\mu=4, \lambda=0.5, \rho=0.2$.} 
\vspace{-5mm}
\label{fig:perf-condition}
\end{figure} 
 
\begin{figure*}[t]
    % \hfill
	\begin{minipage}[b]{0.325\textwidth}
		\centering
			\centerline{\includegraphics[width=\linewidth, keepaspectratio,]{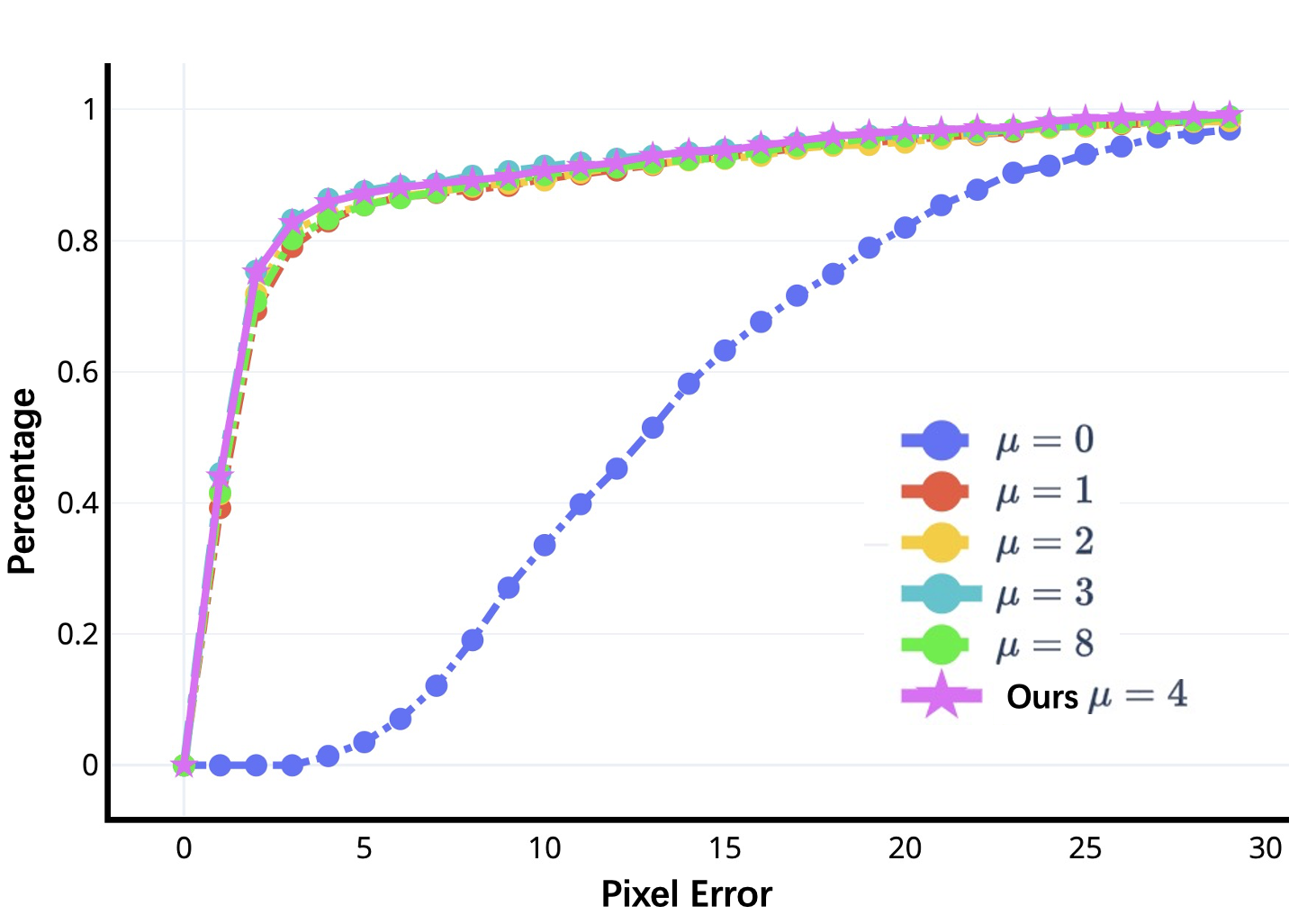}}
			\centerline{Various $\mu$}
	\end{minipage}
	\hfill
	\begin{minipage}[b]{0.325\textwidth}
		\centering
			\centerline{\includegraphics[width=\linewidth,keepaspectratio]{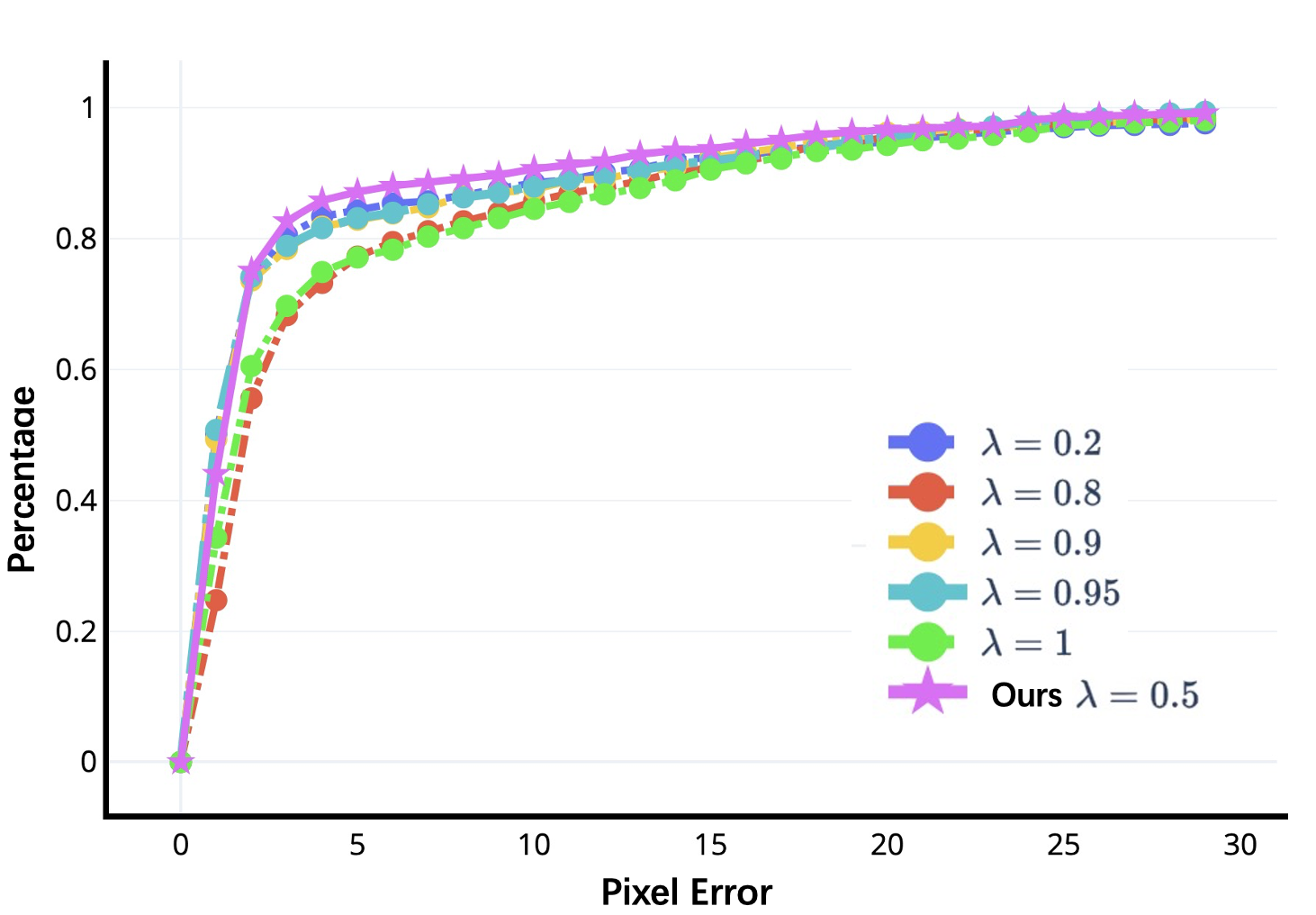}}
			\centerline{Various $\lambda$}
	\end{minipage}
	\hfill
	\begin{minipage}[b]{0.325\textwidth}
		\centering
			\centerline{\includegraphics[width=\linewidth,keepaspectratio]{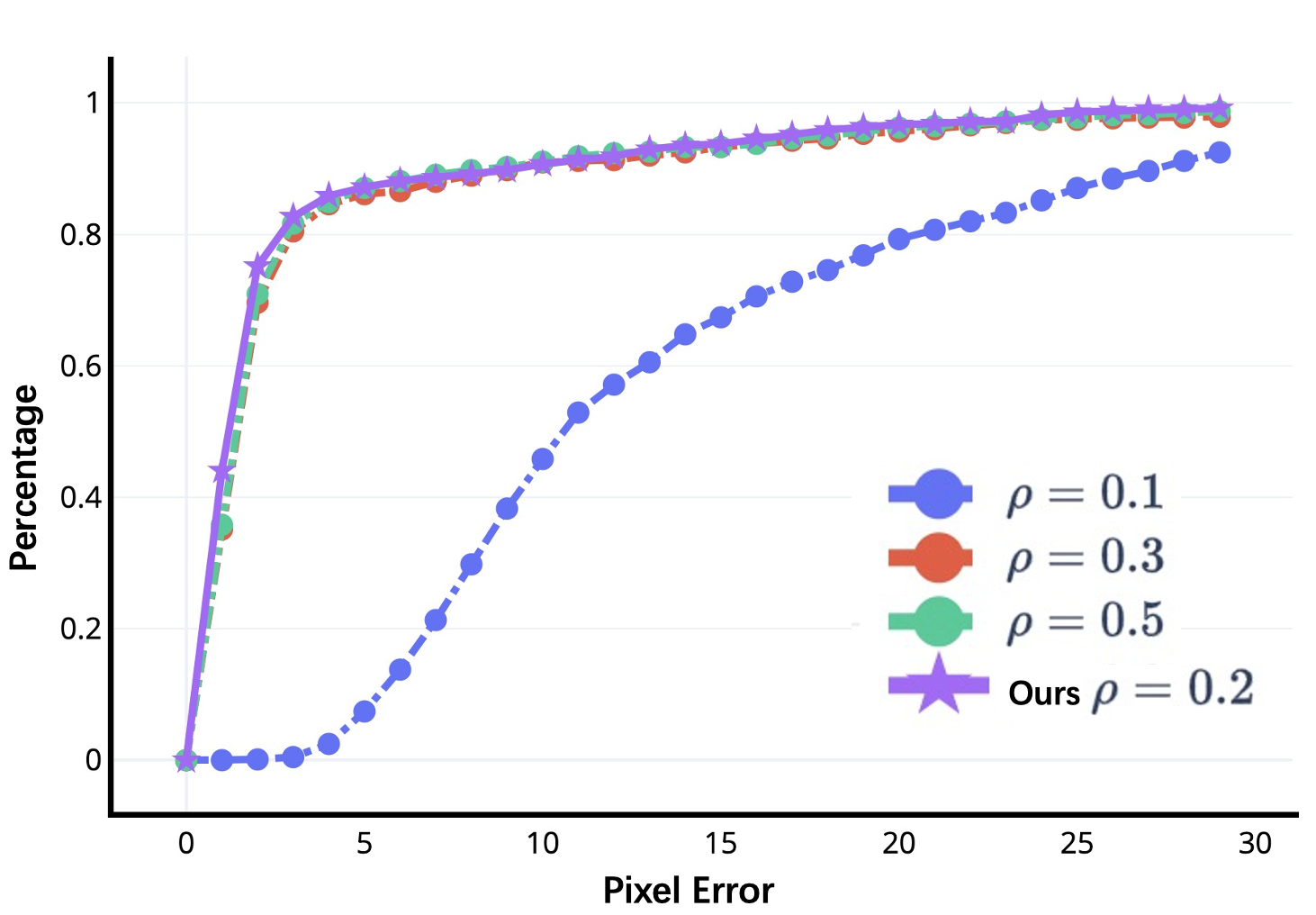}}
			\centerline{Various $\rho$}
	\end{minipage}
	\vspace{-2mm}
    \caption{Pixel error \vs various hyperparameters in Eqs. \ref{eqn:con1}-\ref{eqn:con3} on Google Earth.}
\label{fig:PE-param}
	\vspace{-5mm}
\end{figure*}

% \begin{table}[t]
% \centering
% \caption{ }
% \setlength{\tabcolsep}{2.5pt}{
% \begin{tabular}{c|cccc}
% \toprule
% Method & Recall & RRE & RTE & Time \\

% \midrule
% FGR & 42.7 & 10.6 & 4.08 & 0.31  \\ 
% RANSAC-2M & 66.1 & 8.85 & 3 & 1.39  \\ 
% RANSAC-4M & 70.7 & 9.16 & 2.95 & 2.32  \\ 
% RANSAC-8M & 74.9 & 8.96 & 2.92 & 4.55  \\ 
% Go-ICP & 22.9 & 14.7 & 5.38 & 771  \\ 
% Super4PCS & 21.6 & 14.1 & 5.25 & 4.55  \\ 
% ICP(P2Point) & 6.04 & 18.1 & 8.25 & 0.25  \\ 
% ICP(P2Plane) & 6.59 & 15.2 & 6.61 & 0.27  \\
% DCP & 3.22 & 21.4 & 8.42 & 0.07  \\ 
% PointNetLK & 1.61 & 21.3 & 8.04 & 0.12  \\ 
% DGR & 85.2 & 7.73 & 2.58 & 0.7  \\ 
% DeepHough & 91.4 & 6.61 & 2.08 & 0.46 \\ 
% \midrule
% DCL+PRNet & 10.04 & 3.20 & 8.74 & 0.51 \\
% \bottomrule
% \end{tabular}}
% \label{table:newtable}
% \end{table}

\subsubsection{Iterative Inference at Test Time}
We demonstrate the iterative inference process in Fig. \ref{fig:iteration_demo} on ModelNet40, where failure cases often have few overlaps with each other in point clouds. We also illustrate the inference behavior with and without DLC over iterations in Fig. \ref{fig:iteration}. Recall that PRNet is an RNN, and thus we can see that DLC can consistently improve the performance over iterations to reach a better (local) minimum. The iterative inference can also help PRNet slightly as itself is a recursive network. In terms of running time, we can see that DLC does not introduce extra computational burden to the inference, because we do not modify the architecture but learn better network weights. Also, it is clear that the running time is approximately linear to the number of iterations. In our experiments, we fine tune the number of iterations and report the best performance.

\subsubsection{Robustness}

\bfsection{Unseen Class/Scene Testing.}
Table \ref{table:reliability} lists our comparison results, where ``-'' indicates that we cannot train DCP successfully on 3DMatch due to the limited memory of GPUs. ICP refinement refers to a post-processing step where the network outputs are taken as the initialization for ICP. As we see, DLC can always significantly improve the estimation of rotation matrices on both datasets, but sometimes it has little help in translation estimation such as on ModelNet40. One possible reason is that 3DMatch is a real world dataset with various objects in a single frame and contains more texture information with more points, leading to more challenges in registration than ModelNet40. Note that ICP refinement can further boost the performance in all the cases. 

% In ModelNet40, we find the DLC method can boost the performance for DCP and PRNet. Also, after we add ICP refinement as shown in the previous section, we can enjoy additional enhancement, which is still leading the original DCP and PRNet. In general, 3DMatch is a real world dataset with various objects in a single frame and contains more texture information collected by a RGB-D camera. The DCP method can not handle it. We only provide the experiments based on the PRNet. Our DCL+PRNet gets 10 times improvement compared with PRNet. With the ICP refinement, we also get a 50 
 % \% increment in the MSE(Euler) dropping from 37 to 18.

\bfsection{Transferability.}
To show this property, we use the same setting as PRNet with more challenging scenarios, where both PRNet and DLC+PRNet are pretrained on ModelNet40 clean data and then directly tested on ShapeNetCore test data. Table~\ref{table:transferability} summarizes our results. As we see, DLC can still help PRNet achieve better transferability, in general, especially together with ICP refinement. % the numbers in Table \ref{table:transferability} have similar patterns to those in Table \ref{table:reliability}, and our DLC+PRNet has much better transferability in terms of both rotation and translation.

\bfsection{Scalability.}
To run PRNet and DLC+PRNet, we first use farthest point sampling to down-sample each raw point cloud into 2048 points\footnote{This preprocess is purely due to the characteristics of PRNet, and different backbone networks may need different preprocesses.}. Then following our previous experiments, we managed to run both methods and show our test results in Fig. \ref{fig:kitti}. Clearly, in all the metrics, DLC+PRNet consistently outperforms PRNet.

\subsubsection{State-of-the-art Comparison}
Table \ref{table:modelnet-clean} lists our comparison results with more baselines. Clearly, our DLC and ICP refinement (almost) consistently and significantly improve the performance of DCP and PRNet, and among all the competitors DLC+DCP+ICP works best. Compared with the second best numbers, our approach has reduced the MSE(Euler) and MSE(T) by 22.4\% and 46.7\% relatively.

\subsection{Multimodel Image Alignment}

\bfsection{Datasets.}
We exactly follow the experimental settings in DeepLK \cite{zhao2021deep}. We select an image from each dataset and resize it to $196\times 196$ pixels as input. Then we randomly perturb 4 points in the four corner boxes with a size of $64\times64$ and resize the chosen region to a $128\times 128$ template image. We implement the same data generation strategy on the three different datasets, namely MSCOCO \cite{lin2014microsoft}, GoogleEarth \cite{zhao2021deep}, Google Maps and Satellite (GoogleMap for short) \cite{zhao2021deep}. Please refer to \cite{zhang2023prise} for more details.

\bfsection{Baseline Algorithms.} We train DHM \cite{detone2016deep} and MHN \cite{le2020deep} from scratch with the best hyperparameters. We use the pretrained model for DeepLK\cite{zhao2021deep} directly. In addition, we use the pretrained models for CLKN \cite{chang2017clkn} and fine-tune it on MSCOCO, GoogleEarth, and GoogleMap to fix the domain gap. Also we compare our approach with a classical algorithm SIFT+RANSAC.

\bfsection{Training \& Testing Protocols.}
Following the consistent setting for each dataset in both training and testing, we use the same resolution for the source and target images as the standard datasets. For training, we train our model with the best hyperparameters on each dataset for 10 times with random initialization of network weights. For testing, we conduct evaluation on the PEs under different thresholds. We report our results in terms of mean and standard deviation over 10 trials. %All the experiments are done using an Nvidia RTX6000 GPU server

\bfsection{Implementation.}
We modify the code of DeepLK. Specifically, we adopt the same network architecture and stage-wise network training and inference used in DeepLK. We use the grid search to tune hyperparameters $\mu, \lambda, \rho$, while the rest hyperparameters keep the same as used in DeepLK.
% By default, we set the hyperparameters in Eq. \ref{eqn:problem} as $\mu=4, \lambda=0.5, \rho=0.2$. 
We train each stage in the network for 10 epochs with a batch size of 4, a constant learning rate $10^{-5}$, and weight decay $0.05$. We take Adam as our optimizer.

\bfsection{Evaluation Metric.}
We use the same evaluation metric as in recent works, Success Rate (SR) \vs Pixel Error (PE), to compare the performance of each algorithm. PE measures the average $L_2$ distance between the 4 ground-truth perturbation points and the 4 output point location predictions (without quantization) from an algorithm, correspondingly. Then the percentages of the testing image pairs whose PEs are smaller than certain thresholds, \ie SR, are computed to compare the performance of different approaches.

\bfsection{Results.}
Given the page limit, below we briefly summarize and discuss our results. Overall, our new formula has achieved very similar performance to that in our prior work \cite{zhang2023prise}. Specifically, 
\begin{itemize}[nosep, leftmargin=*]
    \item {\em Impact of Each Constraint in Training:} Fig. \ref{fig:loss} illustrates our comparisons on the training loss at stage 3. We can see that: (i) The LK loss and hinge loss are balanced well at the same scale, no dominant scenarios occurring. All the losses tend to converge over the epochs. (ii) From Fig. \ref{fig:loss}(c), without any hinge loss, training with the original LK loss alone is very easy to overfit. (iii) Condition 1 alone seems less useful to prevent the overfitting, and condition 2 seems most important among all the three conditions that keep the LK loss away from zero. We also illustrate our performance comparison in Fig. \ref{fig:perf-condition} using different combinations of the constraints, and ``ours'' that contains all the constraints in learning leads to the best performance.

    \item {\em Hyperparameter Tuning:} Fig. \ref{fig:PE-param} illustrates the comparisons among different settings. Overall, our approach is very robust to different hyperparameters. Specifically, Fig. \ref{fig:motivation}(a) demonstrates the necessity of strong star-convexity, which makes sense that we expect to learn a unique minimum within a large local region rather than several. Fig. \ref{fig:PE-param}(b) shows that the midpoint choice (\ie $\lambda=0.5$) indeed provides good performance. Fig. \ref{fig:PE-param}(c) shows that when $\rho$ is small, the hinge loss cannot prevent the overfitting, which makes sense. When it increases to a sufficiently large number, the performance will be improved significantly and stably. All the values here for $\rho$ make good balance between the LK loss and the hinge losses.

    \item {\em State-of-the-art Comparison:} We list our comparison results in Table \ref{table:COCO}. We observe that our numbers are almost identical to \cite{zhang2023prise} based on the same data, with only marginal differences that cannot be shown in the table given the precision. 
\end{itemize}

% \begin{table*}[h]
% \centering
% \setlength{\tabcolsep}{3pt}{
% \begin{tabular}{cccccccc}
% \toprule
% Method&  PE<0.1& PE<0.5& PE<1& PE<3& PE<5& PE<10& PE<20\\
% \midrule
% DHM & 0.00 & 0.02 & 1.46 & 2.65 & 5.57 & 25.54 & 90.32 \\
% MHN & 0.00 & 3.42 & 4.56 &5.02 & 8.99 & 59.90 & 93.77 \\
% SIFT+RANSAC&0.18&3.42&8.97&23.09 & 41.32&50.36& 59.88 \\
% CLKN& \textbf{0.27} &2.88&3.45&4.24 & 4.32 & 8.77& 75.00 \\
% DeepLK& 0.00 & 3.50& 12.01& 70.20& 84.45& 90.57& 95.52\\
% \textbf{PRISE}& 0.24 $\pm$ 1.83 & \textbf{25.44} $\pm$ 1.21& \textbf{53.00} $\pm$ 1.54& \textbf{82.69} $\pm$ 1.07& \textbf{87.16} $\pm$ 1.09& \textbf{90.69} $\pm$ 0.73& \textbf{96.70} $\pm$ 0.54\\
% \bottomrule
% \end{tabular}}
% \caption{Performance comparison on Google Earth dataset.}
% \label{table:GE}
% \end{table*}

% \begin{table*}[h]
% \centering
% \setlength{\tabcolsep}{3pt}{
% \begin{tabular}{cccccccc}
% \toprule
% Method& PE<0.1& PE<0.5&PE<1& PE<3& PE<5& PE<10& PE<20\\
% \midrule
% DHM & 0.00 & 0.00 & 0.00 & 1.20 & 3.43 & 6.99 & 78.89 \\
% MHN & 0.00 & 0.34 & 0.45 &0.50 & 3.50 & 35.69 & 93.77 \\
% SIFT+RANSAC&0.00&0.00&0.00&0.00 & 0.00 &2.74& 3.44 \\
% CLKN&0.00&0.00&0.00&1.57 & 1.88 &8.67& 22.45 \\
% DeepLK& 0.00& 2.25&16.80& 61.33& 73.39& 83.20& 93.80\\
% \textbf{PRISE}& \textbf{17.47} $\pm$ 2.44
% &\textbf{48.13}$ \pm$ 12.00& \textbf{56.93} $\pm$ 3.45&\textbf{76.21} $\pm$ 2.43& \textbf{80.04} $\pm$ 5.55& \textbf{86.13} $\pm$ 0.47& \textbf{94.02} $\pm$ 1.66\\
% \bottomrule
% \end{tabular}}
% \caption{Performance comparison on Google Maps and Satellite dataset.}
% \label{table:GM}
% \end{table*}

\begin{table*}[h]
\centering\small
\caption{Performance comparison on MSCOCO, GoogleEarth, and GoogleMap.}
\vspace{-2mm}
\setlength{\tabcolsep}{1pt}{
\begin{tabular}{cc|ccccccc}
\toprule
Dataset&Method& PE<0.1& PE<0.5& PE<1& PE<3& PE<5& PE<10 & PE<20\\
\hline
\multirow{6}{*}{\rotatebox[origin=c]{90}{MSCOCO}} & SIFT+RANSAC&0.00&4.70&68.32& 84.21& 90.32 &95.26&96.55 \\
& SIFT+MAGSAC \cite{barath2019magsac} &0.00& 3.66&76.27& 93.26& 94.22& 95.32& 97.26\\
& LF-Net \cite{ono2018lf} &5.60&8.62&14.20& 23.00& 78.88 &90.18&95.45 \\
& LocalTrans \cite{shao2021localtrans} &38.24&\textbf{87.25}&96.45& 98.00& 98.72 &99.25&\textbf{100.00}\\
&DHM\cite{detone2016deep}&0.00&0.00&0.87&3.48&15.27&98.22&99.96\\
&MHN\cite{le2020deep}&0.00&4.58&81.99&95.67&96.02&98.45&98.70\\
&CLKN\cite{chang2017clkn}&35.24& 83.25&83.27& 94.26& 95.75& 97.52& 98.46\\
&DeepLK\cite{zhao2021deep}&17.16& 72.25&92.81& 96.76& 97.67& 98.92&99.03\\
&\textbf{PRISE}& \textbf{52.77} $\pm$ 12.45&
83.27 $\pm$ 5.21&
\textbf{97.29} $\pm$ 1.82&\textbf{98.44} $\pm$ 1.06&\textbf{98.76} $\pm$ 0.08& \textbf{99.31} $\pm$ 0.53&
99.33 $\pm$ 1.84\\
\hline
\multirow{6}{*}{\rotatebox[origin=c]{90}{GoogleEarth}}&SIFT+RANSAC&0.18&3.42&8.97&23.09 & 41.32&50.36& 59.88 \\
&SIFT+MAGSAC \cite{barath2019magsac}&0.00& 0.00&1.88& 2.70& 3.25& 10.03& 45.29\\
&DHM\cite{detone2016deep} & 0.00 & 0.02 & 1.46 & 2.65 & 5.57 & 25.54 & 90.32 \\
&MHN\cite{le2020deep} & 0.00 & 3.42 & 4.56 &5.02 & 8.99 & 59.90 & 93.77 \\
&CLKN\cite{chang2017clkn}& \textbf{0.27} &2.88&3.45&4.24 & 4.32 & 8.77& 75.00 \\
&DeepLK\cite{zhao2021deep}& 0.00 & 3.50& 12.01& 70.20& 84.45& 90.57& 95.52\\
&\textbf{PRISE}& 0.24 $\pm$ 1.83 & \textbf{25.44} $\pm$ 1.21& \textbf{53.00} $\pm$ 1.54& \textbf{82.69} $\pm$ 1.07& \textbf{87.16} $\pm$ 1.09& \textbf{90.69} $\pm$ 0.73& \textbf{96.70} $\pm$ 0.54\\
\hline
\multirow{6}{*}{\rotatebox[origin=c]{90}{GoogleMap}}&SIFT+RANSAC&0.00&0.00&0.00&0.00 & 0.00 &2.74& 3.44 \\
&SIFT+MAGSAC \cite{barath2019magsac}&0.00& 0.00&0.00& 0.00& 0.00& 0.15& 2.58\\
&DHM\cite{detone2016deep} & 0.00 & 0.00 & 0.00 & 1.20 & 3.43 & 6.99 & 78.89 \\
&MHN\cite{le2020deep} & 0.00 & 0.34 & 0.45 &0.50 & 3.50 & 35.69 & 93.77 \\
&CLKN\cite{chang2017clkn}&0.00&0.00&0.00&1.57 & 1.88 &8.67& 22.45 \\
&DeepLK\cite{zhao2021deep}& 0.00& 2.25&16.80& 61.33& 73.39& 83.20& 93.80\\
&\textbf{PRISE}& \textbf{17.47} $\pm$ 2.44
&\textbf{48.13}$ \pm$ 12.00& \textbf{56.93} $\pm$ 3.45&\textbf{76.21} $\pm$ 2.43& \textbf{80.04} $\pm$ 5.55& \textbf{86.13} $\pm$ 0.47& \textbf{94.02} $\pm$ 1.66\\
\bottomrule
\end{tabular}}
\vspace{-3mm}
\label{table:COCO}
\end{table*}

\section{Conclusion}
In this paper, we propose a novel Deep Loss Convexification (DLC) method to help neural networks learn convex-like loss landscapes \wrt predictions at test time. Our basic idea is to facilitate the convergence of an iterative algorithm at test time based on such convex-like shapes with a certain guarantee. To this end, we propose a specific training algorithm by introducing star-convexity as the constraint on the loss landscapes into the original learning problem. Equivalently, this leads to three extra hinge losses appended to the original training loss. To solve this problem efficiently and effectively, we further propose a sampling-based algorithm to mitigate the challenges in the optimization with random samples (\ie discretizing the continuous spaces). In this way, we can train the networks with traditional deep learning optimizers. We also theoretically analyze the near-optimality property with star-convexity based DLC. We finally demonstrate the superior performance of DLC in three different tasks of recurrent neural networks, 3D point cloud registration, and multimodel image alignment.

\ifCLASSOPTIONcaptionsoff
  \newpage
\fi

% \newpage

\bibliographystyle{IEEEtran}
\bibliography{egbib}

\vskip -35pt plus -1fil

\begin{IEEEbiography}[{\includegraphics[width=1in,height=1.25in,clip,keepaspectratio]{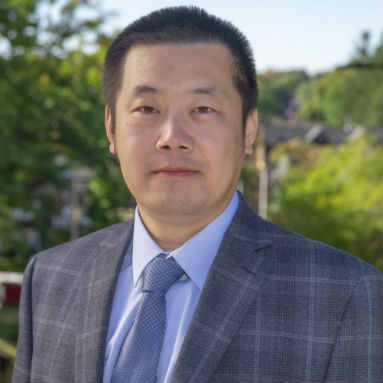}}]{Ziming Zhang} is an assistant professor at Worcester Polytechnic Institute (WPI). Before joining WPI he was a research scientist at Mitsubishi Electric Research Laboratories (MERL) in 2017-2019. Prior to that, he was a research assistant professor at Boston University in 2016-2017. Dr. Zhang received his PhD in 2013 from Oxford Brookes University, UK, under the supervision of Prof. Philip H. S. Torr. His research interests lie in computer vision and machine learning. He won the R\&D 100 Award 2018.
\end{IEEEbiography}

\vskip -30pt plus -1fil

\begin{IEEEbiography}[{\includegraphics[width=1in,height=1.25in,clip,keepaspectratio]{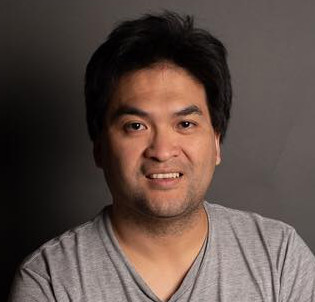}}]{Yuping Shao} received his M.S. in Electrical and Computer Engineering from Worcester Polytechnic Institute, USA in 2023. His research interests lie in 3D point cloud registration and medical imaging applications.

\end{IEEEbiography}

\vskip -50pt plus -1fil

\begin{IEEEbiography}[{\includegraphics[width=1in,height=1.25in,clip,keepaspectratio]{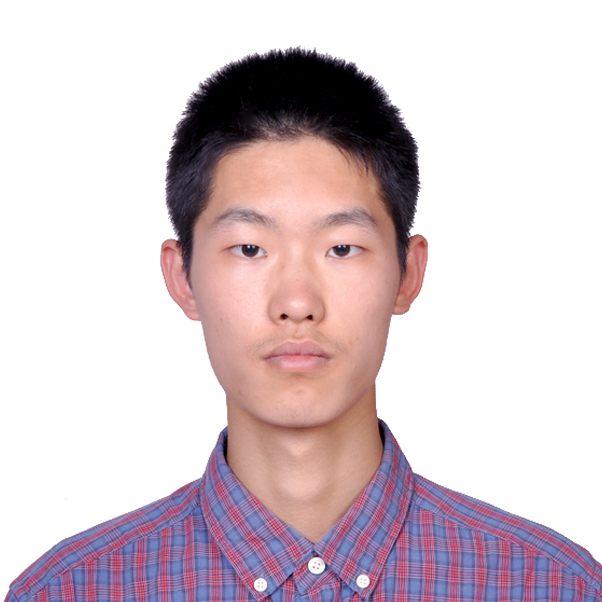}}]{Yiqing Zhang} received his B.S. degree in Statistics from Beijing Technology and Business University, China in 2019 and M.S. degree in Data Science from Worcester Polytechnic Institute, USA in 2021 where he is currently a Ph.D student working on multi-phase learning for clinical trial design. His current research interest is Knowledge Databases (and Graphs) for Drug Discovery.

\end{IEEEbiography}

\vskip -50pt plus -1fil

\begin{IEEEbiography}[{\includegraphics[width=1in,height=1.25in,clip,keepaspectratio]{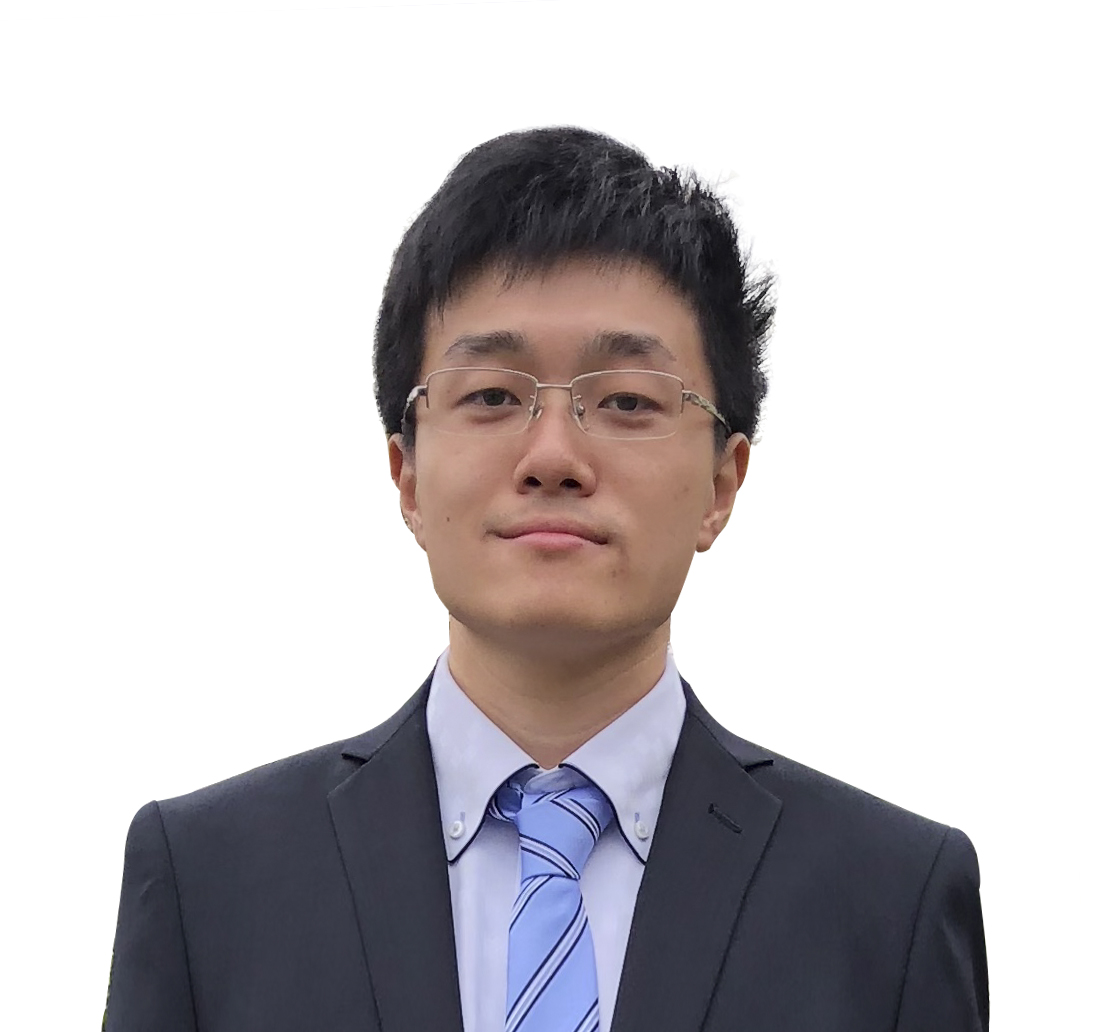}}]{Fangzhou Lin}
received his M.S. and Ph.D. in information sciences and data science from Tohoku University, Japan, Sendai, in 2021 and 2024.  His research interests include 3D point cloud and medical imaging applications.

\end{IEEEbiography}

\vskip -50pt plus -1fil

\begin{IEEEbiography}[{\includegraphics[width=1in,height=1.25in,clip,keepaspectratio]{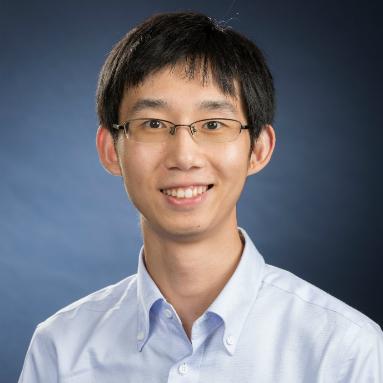}}]{Haichong Zhang} is an Associate  Professor at Worcester Polytechnic Institute (WPI). He is the founding director of the Medical Frontier Ultrasound Imaging and Robotic Instrumentation (FUSION) Laboratory. The research in his lab focuses on the interface of medical imaging, sensing, and robotics, developing robotic-assisted imaging systems as well as image-guided robotic interventional platforms. Dr. Zhang received his B.S. and M.S. in Human Health Sciences from Kyoto University, Japan, and subsequently earned his M.S. and Ph.D. in Computer Science from Johns Hopkins University.

\end{IEEEbiography}

\vskip -40pt plus -1fil

\begin{IEEEbiography}[{\includegraphics[width=1in,height=1.25in,clip,keepaspectratio]{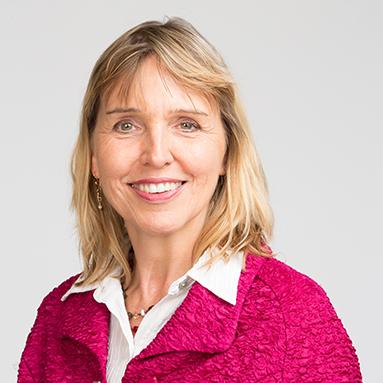}}]{Elke Rundensteiner} is the founding director of the Interdisciplinary Data Science Program, WPI and professor in computer science. Her research, focused on big data management, machine learning, and visual analytics, has been funded by agencies from NSF, NIH, DOE, FDA, to DARPA and resulted in more than 400 publications, patents, and numerous honors and awards. She holds leadership positions in the Big Data field.

\end{IEEEbiography}

\end{document}